\documentclass[twoside,11pt]{article}
\usepackage{jair, theapa, rawfonts}

\usepackage{balance} 
\usepackage{algorithmic}
\usepackage{algorithm}
\usepackage{booktabs} 
\usepackage{amsmath}
\usepackage{subfigure}
\usepackage{mathtools}

\usepackage{amsthm}
\usepackage{comment}
\newtheorem{theorem}{Theorem}

\newtheorem{definition}{Definition}[section]
\ShortHeadings{Sequential Choice Bandits with Feedback }
{}
\firstpageno{1}

\begin{document}

\title{Sequential Choice Bandits with Feedback for Personalizing users' experience}

\author{\name Anshuka Rangi \email arangi@ucsd.edu \\
       \name Massimo Franceschetti \email massimo@ece.ucsd.edu \\
       \addr University of California, San Diego,\\
       9500 Gilman Drive,
        La Jolla, CA, USA \\
       \AND
       \name Long Tran-Thanh \email long.tran-thanh@warwick.ac.uk \\
       \addr  University of Warwick, \\ 
       Coventry CV4 7AL, United Kingdom}


\maketitle

\begin{abstract}
In this work, we study \emph{sequential choice bandits with feedback}.
We propose bandit algorithms  for a platform that personalizes users' experience to maximize its rewards. 
For each action directed to a given user, the platform is given a positive reward, which is a non-decreasing function of the action,  if this action is below the user's threshold. 
Users are equipped with a patience budget, 
and actions that are above the threshold decrease  the user's patience.
When all patience is lost, the user abandons the platform.
The platform   attempts to learn the thresholds of the users in order to maximize its rewards, based on two different feedback models describing the information pattern available to the platform at each action.   
We define a notion of regret   by    determining the best action to be taken when the platform knows that the user's threshold is in a given interval. 
We then propose   bandit algorithms for the  two feedback models
and show that upper and lower bounds on the regret  are of the order of $\tilde{O}(N^{2/3})$ and $\tilde\Omega(N^{2/3})$, respectively,  where $N$ is the total number of users.
Finally, we show that the waiting time of any  user before receiving a personalized experience is uniform in $N$.
\end{abstract}

\section{Introduction}
Many  companies  like Amazon, Facebook, YouTube and Netflix face the problem of personalizing users (or customers) experience to maximize their rewards. Personalization is achieved by  learning  users' preferences over time, and leveraging this information to recommend relevant products. However, marketing fatigue, which refers to exposure to irrelevant suggestions or excessive recommendation, may backfire. It results in user's dissatisfaction and eventual abandonment of the platform.
Consider for example news feeds, e-stores, on-line blogs, or mobile apps. For all of these platforms, there is a value in displaying the advertisement of new products, and sending emails or notifications related to promotions. However, displaying (or sending) too many of them can impact the users negatively, and increases the risk of abandonment.  Based on a survey \cite{harvard}, the cost of user acquisition is estimated to be 5 to 25 times higher than keeping the existing user. Hence, controlling the fatigue (or unsatisfactory actions) is crucial. This can be cast into a decision problem where initially ``the more is the better,'' while eventually users lose their patience and do not tolerate additional unsatisfactory actions.

Platforms typically receive feedback signals from their customers, which can be used to avoid or delay abandonment. 
For example, a user can provide a \emph{ negative feedback} by closing the displayed advertisements, posts, or articles; or marking them as unrelated (as for example on Facebook and news blogs); or by disliking the recommended item (as for example on Amazon, YouTube and Netflix); or by ignoring  promotional emails and notifications. 
On the other hand, a \emph{positive feedback} would correspond to visiting a recommended product in the promotion email, clicking the displayed advertisement, or watching the recommended video.

A similar problem arises in human resource management in corporate environments. In particular, companies typically push their employees to work as hard as possible within their own comfort zones (e.g., within the employee's accepted range of work-life balance). However, the boundaries of these comfort zones is not known a priori, and even the employees may not be fully conscious of their own boundaries. On the other hand, if the comfort boundary is crossed, it will cause annoyance and grievance, and crossing it too many times will lead to the separation of the employee.

Another example is a negotiation between two agents: Agent 1 and Agent 2. In a negotiation, there is a maximum amount that  Agent 1 is ready to pay for the services. If Agent 2 presents too many options above the maximum amount, then Agent 1 might not be willing to negotiate anymore, and there would be no settlement. 

\subsection{Sequential Choice Bandits with Feedback}
Motivated by these examples, in this paper we consider a \emph{sequential choice bandit problem with feedback}, where a platform interacts with $N$ users   and performs its actions based on on-line feedback, it may receive from the users. 

The interaction between the users and the platform is based on  a \emph{Thresholded  Learning Model} (TLM), where  at the $t^{th}$ round of interaction with user $n$, the platform performs an action $y_{n}(t)\in[0,1]$, and receives a reward $r(y_{n}(t))>0$ if the action $y_{n}(t)$ is below a threshold $\theta_{n} \in (0,1)$. The  thresholds  are drawn independently from a distribution $F$ unknown to the platform. The reward  is a non-decreasing function  of the action and is identical for all users.
Users are equipped with a patience budget, whose initial value $B$ is uniform for all users and known to the platform, 
and actions that are above the threshold decrease  the user's budget by a depletion cost of one unit.
When all of a user's patience is lost, the user abandons the platform. In this model, the users' thresholds correspond to the boundary between satisfactory and unsatisfactory actions. 
Compared to the classic Multi-Armed Bandit (MAB) setting, a key challenge in our setting is that  the threshold of the user is never  revealed to the platform. All of our results  can be easily extended to the case where thresholds are  revealed, for example, at time of abandonment through a feedback questionnaire; or a noisy version of them becomes known. By performing an extra layer of optimization, we 
can also account for user-dependent initial budgets $B(n)$ and user-dependent patience depletion costs $c(n) \geq 1$. 
\subsection{Feedback models}
We consider 
two different feedback models describing the information pattern available to the platform at each action. In a \emph{hard feedback} model, 
at each action it is revealed whether this action was above or below the threshold. Namely, the platform observes the feedback $\mathbf{1}(y_{n}(t)\leq \theta_{n})$. This  corresponds to immediately observing the   reward of each action. This model can also be considered as a deterministic feedback model. 
In a \emph{soft feedback} model, at each action below the threshold, the feedback is revealed with probability $p_1 \in (0,1)$, and at each action above the threshold, the feedback is revealed with probability $p_2\in (0,1)$. Namely, if  $y_{n}(t) \leq \theta_{n}$,  then the platform observes $\mathbf{1}(y_{n}(t)\leq \theta_{n})$  with probability $p_1$, and if  $y_{n}(t) > \theta_{n}$, then the platform observes $\mathbf{1}(y_{n}(t)\leq \theta_{n})$  with probability $p_2$. In practice, $p_1$ corresponds to  the user probabilistically providing positive feedback by visiting the recommended products if the user likes the recommendations, or the employee likes certain recently allocated tasks.  In practice, $p_2$ corresponds to the user probabilistically providing negative feedback by disliking   advertisements and recommended products or requesting to reduce the number of notifications if the user is negatively impacted by them. 
The \emph{soft feedback model} corresponds to not being able to immediately observe the reward  of  each  action,  but  to  provide partial probabilistic feedback. Note that if $p_1=1$, then the soft feedback model reduces to hard feedback model since the absence of any feedback would imply that $y_{n}(t)> \theta_n$. Likewise, if $p_2=1$, then the soft feedback model  reduces to hard feedback model since the absence of any feedback would imply that $y_{n}(t)\leq \theta_n$. 

\subsection{Objective function}
The objective of the platform is to maximize its total discounted reward over all $N$ users, namely 

\small{
\begin{equation}\label{eq:OptimizationProblem}
\begin{aligned}
 \max_{\{\{y_n(t)\}_{t=1}^{\infty}\}_{n=1}^{N}}\mathbf{E}\Big[\sum_{n}\sum_{t=1}^{T_n}\gamma^{t-1}r(y_n(t))\mathbf{1}(y_n(t)\leq \theta_{n})\Big]\\
 \mbox{s.t.  } \forall n\in[N]: \mathbf{P}\left(\sum_{t=1}^{T_{n}}\mathbf{1}(y_{n}(t)> \theta_{n})\leq B  \right)=1,
\end{aligned}
\end{equation}
}
\normalsize
where $\gamma \in [0,1)$ is a discount factor, the random variable $T_{n}$ is the total number of rounds of interaction of user $n$ with the platform and  $\{y_{n}(t)\}_{t=1}^{T_{n}}$ denotes the sequence of actions for user $n$. 

\subsection{Key novel aspects of the problem}
First, in contrast with the classical MAB setting, the distribution $F$ of the users' thresholds cannot be learned from the samples $\{\theta_n\}_{n=1}^N$ because the user's threshold $\theta_n$ is never revealed to the platform. Second, unlike Bandits with Knapsack, where the remaining budget (or optimization constraint) is observed by the learner  at all times, in our setting the remaining patience  of the user remains unobserved. Finally, unlike combinatorial bandits where the learner chooses the entire set of actions at the starting of each round, in our setting the actions presented to the user are chosen sequentially, and the next action in the sequence depends on all the past feedback  received from the user, and on the feedback model adopted. This also implies that the choice of the optimal sequence of actions will be different for  the two feedback models. 
\subsection{Contributions}
We propose a \emph{sequential choice bandit} model that captures the  sequential interaction between the users and the platform, including abandonment. Additionally, it incorporates two different feedback models observed in real systems.

To define the regret, we first determine in Section~\ref{sec:optimalPolicies} the best next action in the sequence for the platform, when it knows that the threshold of user $n$ is in an uncertainty interval (UI) $[\ell,u]$. Given an oracle having access to  the distribution $F$, feedback probabilities $p_1$ and $p_2$, and the residual patience $B_r(n)=B_r$ of the user $n$, the platform can present the UI to the oracle, and the oracle returns the best next action in the sequence. This best action is described in terms of a Markov Decision Process (MDP).  Comparing the performance of our learning algorithms with the performance of the platform having access to the oracle, 
we then define in Section~\ref{delta} a $\delta$-policy, which is an approximation of the optimal MDP strategy used by the oracle, and we  introduce the notion of $\delta$-regret, which is the expected difference between the total rewards of the $\delta$-policy and the total rewards of the platform's learning algorithm for $N$ users. 

We   propose learning algorithms for  the two feedback models. Unlike the oracle, these learning algorithms do not have any knowledge of distribution $F$, feedback probabilities $p_1$ and $p_2$, and of the residual patience $B_r$. We show that  the $\delta$-regret of these algorithms is at most $\tilde{O}(N^{2/3})$.  We also show that there exists a distribution $F$ and a reward function $r(.)$ such that the $\delta$-regret of any learning algorithm is at least $\tilde{\Omega}(N^{2/3})$. Thus, our proposed algorithms are order optimal up to a logarithmic factor in $N$.  Additionally, we also establish that both the probabilistic positive feedback corresponding to $p_1$ and the probabilistic negative feedback corresponding to $p_2$ can be  utilized to obtain order optimal regret. 

We then consider another important performance measure that is the waiting time of the user before receiving a personalized experience.
We show that for our algorithms the waiting time of any user is $O(B)$, uniformly in  $N$. Since practical systems handle a number of users in the order of millions,  we conclude that our algorithms are also practical.

\section{Related Work}
The idea of considering a sequential choice Bandit model with feedback is quite unique, and  we are only aware of limited works in a similar setting \cite{schmit2018learning,LuCustomer2017,cao2019dynamic}. The work in    \cite{schmit2018learning} studies a setting  where users have no patience, namely $B=0$, and crossing a user's threshold always leads to the immediate abandonment of the platform.  For this special case, the optimal strategy is a fixed constant action for all $N$ users. This eliminates the need for both personalization, and   feedback utilization from the users. 
In the case of a fixed constant action, the classic Upper-Confidence Bound (UCB) algorithm in ~\cite{bubeck2012regret} is order optimal, and  the waiting time is $O(N)$. Unlike \cite{schmit2018learning}, in our work users are tolerant towards unsatisfactory actions, and this enables both personalization and feedback utilization. This also opens up the possibility of studying different users' feedback models. 
The optimal solution is not a constant action, which increases the complexity of the learning process. 
Finally, in our work, the waiting time  is $O(B)$, uniform in $N$. 

In \cite{LuCustomer2017}, authors study user-platform interaction with abandonment, but compared to \cite{schmit2018learning} they  further restrict the platform's actions to safe and risky action, leading  to different results. 

 The work in \cite{kveton2015cascading,cao2019dynamic,cao2019sequential,chen2020revenue} studies  variants of \emph{sequential choice bandit model without feedback}. Unlike our setting, 
 the sequence of action is pre-determined at the arrival of each user, independently from the user's feedback. Hence, both the optimal sequence of actions and the bandit algorithms in the two settings \emph{with feedback} and \emph{without feedback} are different.
 Additionally, in  \cite{cao2019dynamic}, the probability of abandonment at each time step is only dependent on the current action in the sequence.  In contrast, in our setting, the probability of abandonment at each time step is a function of all the past actions presented to the user and is incorporated in the patience budget of the user. Finally, unlike our setting where $\theta_n$ is never revealed and the feedback is probabilistic,  in  \cite{cao2019dynamic}, the reward corresponding to each action is immediately observed by the platform. 





Our setting seems also to be  related to Reinforcement Learning (RL) as the optimal action at each step of interaction is  given by a MDP~\cite{sutton2018reinforcement}. In RL, the state transition dynamics, which correspond to the probability of the next state given the current state and action, is the same across different rounds (commonly referred as episodes in the literature).
In contrast, in our setting, since each user has a different threshold, the state transition dynamics (which depend on the remaining budget and UI) are different for each user. This corresponds to having one single episode to learn  in RL. 
We also point out that theoretical results on RL algorithms have been sparse and mostly  limited to MDPs with discrete state  and action spaces  ~\cite{azar2017minimax,jin2018q}. For MDPs with continuos spaces, a common solution is to discretize the space, and use the corresponding algorithms developed in a discrete setting. However, this approach leads to linear regret bounds ~\cite{azar2017minimax,jin2018q}. Our idea of $\delta$-policy and $\delta$-regret can in principle lead  to  policies with sub-linear regret bounds in RL for continuous MDPs, although this is not investigated  here. 

Our setting is also different from Budgeted Multi-Armed Bandits (BMAB) \cite{ding2013multi,xia2016budgeted,rangi2018multi,rangi2019unifying}. In BMAB, at each round  the learner performs a single action, observes the reward, and pays a cost for performing this action. This cycle is repeated $T$ times and the total budget of performing actions over $T$ rounds is bounded by a constant, which restricts the exploration of different actions. 
In our setting, the platform  performs a sequence of actions for each user $n$, whose length is random and depends on  $B$ and the choice of actions. This occurs for all $N$ users, an analogue of $T$ in BMAB, and there is no restriction on the exploration of different actions across different users. In other words, the budget $B$ determines the selection of one ``super-action,'' consisting of a sequence of actions, at each round, and does not restrict the exploration of actions across $N$. 

\section{Optimal Personalized Strategy}\label{sec:optimalPolicies}
We now describe the best action for the platform, when the threshold $\theta_n$ is in an  UI $[\ell,u]$, in terms of MDP which can be solved by an oracle having access to the  distribution $F$, feedback probabilities $p_1$ and $p_2$, and the residual patience $B_r$ of each user, for  the two feedback models. We will then compare the performance of the oracle and our learning algorithms, which has no knowledge of $F$, $p_1$, $p_2$ and $B_r$. 

The platform uses the user's feedback to reduce the UI containing the user's threshold. For example, initially  $\theta_{n}\in [0,1]$. After performing $y_n(t)=0.3$ and receiving a positive feedback  $\mathbf{1}(y_{n}(t)\leq \theta_n)=1$, the platform learns that  $\theta_{n}\in [0.3,1]$. Next, if $y_n(t+1)=0.6$ and negative feedback $\mathbf{1}(y_{n}(t)\leq \theta_n)=0$ is received, then  $\theta_{n}\in [0.3,0.6]$.

The MDP provides an optimal policy for each user, where the policy is a mapping from the state space to the action space of the MDP. The state space, represented as $\{(\ell,u,B_{r})_n\}$, consist of the UI $[\ell,u]$ from the platform and the residual budget $B_r$ from  user $n$,  where $\ell$ and $u$ indicate lower and upper bounds on the threshold $\theta_{n}$ of the user $n$. Thus, the state space is continuous with respect to $\ell$ and $u$, and is discrete with respect to $B_{r}$. The action space is the interval $[0,1]$. The MDP provides the best action for the platform, when the threshold $\theta_n$ is in a UI $[\ell,u]$, using its knowledge of the distribution $F$ and the residual patience $B_r$ from user $n$.


In the \emph{hard feedback model}, the platform always receives a feedback $\mathbf{1}(y_{n}(t)\leq \theta_{n})$ for each action $y_{n}(t)$, and therefore is able to reduce the UI $[\ell,u]$  at every interaction. 
For the generic user, the optimal policy given by the MDP performs an action $y^*(\ell,u,B_{r})$ in  state $(\ell,u,B_{r})$ and its expected reward (or value function) is $V^*(\ell,u,B_{r})$ in state $(\ell,u,B_{r})$, where
\cite{puterman1990markov}
\begin{equation}\begin{split}\label{eq:optimalAction}
    &y^*(\ell,u,B_{r})=\mbox{argmax}_{\ell\leq y\leq u} V(\ell,u,B_{r},y),\\ &V^*(\ell,u,B_{r})=\mbox{max}_{\ell\leq y\leq u} V(\ell,u,B_{r},y),
\end{split}\end{equation}
 and $V(\ell,u,B_{r},y)$ is the maximum expected reward  if action $y$ is performed in state $(\ell,u,B_{r})$, namely 
\begin{equation}\label{eq:expectedReward}\begin{split}
       V(\ell,u,B_{r},y)&= \frac{F(u)-F(y)}{F(u)-F(\ell)}(r(y)+\gamma V^*(y,u,B_{r}))+\frac{F(y)-F(\ell)}{F(u)-F(\ell)}\gamma V^*(\ell,y,B_{r}-1).
\end{split}
\end{equation}
Now, we briefly explain the intuition behind $V(\ell,u,B_{r},y)$ in (\ref{eq:expectedReward}).  In state $(\ell,u,B_r)$, the expectation is computed by considering two mutually exclusive events: $A_{1}=\{y< \theta_{n}\leq u\}$ and $A_{2}=\{\ell\leq \theta_{n}\leq y\}$, and their
conditional probabilities 
\begin{equation}\begin{split}\label{eq:conditionProb1}
    &P(A_{1}|\ell\leq \theta_{n}< u)=\frac{F(u)-F(y)}{F(u)-F(\ell)},\\ &P(A_{2}|\ell\leq \theta_{n}< u)=\frac{F(y)-F(\ell)}{F(u)-F(\ell)}.
\end{split}\end{equation}
If $A_{1}$ occurs, the platform receives the reward $r(y)$ and pays no depletion cost. Also, the feedback is positive i.e. $y\leq \theta_{n}$, the UI  is updated from $[\ell,u]$ to $[y,u]$, and the new state of the MDP is $(y,u,B_{r})$. Likewise, if $A_{2}$ occurs, the platform does not receive any reward, pays the depletion cost of one unit and updates UI to $[\ell,y]$ using the negative feedback.
To conclude, $  V^*(\ell,u,B_{r})$ is the value function of the state $(\ell,u,B_{r})$ and the optimal policy chooses the action with the highest expected reward at each state. Unlike the special case of  \cite{schmit2018learning}, the closed form solution of our MDP cannot be found, and the platform chooses a different action for each user depending on its feedback.

In  the \emph{soft feedback} case,  if $y_{n}(t)\leq \theta_n$, the platform observes the feedback  $\mathbf{1}(y_{n}(t)\leq \theta_{n})$  with probability $p_1$, and if $y_{n}(t) > \theta_{n}$, then the platform observes the feedback $\mathbf{1}(y_{n}(t)\leq \theta_{n})$  with probability $p_2$. 
For the generic user, the optimal action $y^*(\ell,u,B_{r})$ and value function $V^*(\ell,u,B_{r})$ in state $(\ell,u,B_r)$ are 
\begin{equation}\label{eq:negativeSoft}\begin{split}
    &y^*(\ell,u,B_{r})=\mbox{argmax}_{\ell\leq y\leq u}V(\ell,u,B_{r},y),\\ &V^*(\ell,u,B_{r})=\mbox{max}_{\ell\leq y\leq u} V(\ell,u,B_{r},y),
    \end{split}
\end{equation}
where
\begin{equation}\label{eq:SoftUpper}
\begin{split}
      V(\ell,u,B_{r},y)\hspace{-1pt}&=\hspace{-1pt}\frac{F(u)\hspace{-1pt}-\hspace{-1pt}F(y)}{F(u)-F(\ell)}\hspace{-1pt}\big(r(y)\hspace{-1pt}+\hspace{-1pt}\gamma (1-p_1)V^*(\ell,u,B_{r})+\gamma p_1 V^*(y,u,B_{r})\big)\hspace{-1pt}\\
      &+\hspace{-1pt}\frac{F(y)-F(\ell)}{F(u)-F(\ell)}\gamma \big(p_2 V^*(\ell,y,B_{r}\hspace{-1pt}-\hspace{-1pt}1)+\hspace{-1pt}(1\hspace{-1pt}-\hspace{-1pt}p_2)V^*(\ell,u,B_{r}\hspace{-1pt}-\hspace{-1pt}1)\big).
\end{split}
\end{equation}
Given $A_1$, with probability $p_1$, the platform receives a positive feedback, and the state of the MDP becomes $(y,u,B_{r})$. With probability $1-p_1$, there will be no change in the lower bound of the UI and the state of the MDP will remain same, i.e. $(\ell,u,B_{r})$. 
Given $A_{2}$, with probability $p_2$, the platform receives a negative feedback, and the state of the MDP becomes $(\ell,y,B_{r}-1)$. With probability $1-p_2$, there will be no change in the upper bound of the UI and the state of the MDP will be $(\ell,u,B_{r}-1)$. 
If $p_1=1$, then the model reduces to the hard feedback model as the absence of positive feedback implies $A_2$. Likewise, if $p_2=1$, then the model reduces to the hard feedback model as the absence of negative feedback implies $A_1$. 


\section{The $\delta$ -regret}\label{delta}
We now look at the problem from a learning perspective. We introduce the notion of $\delta$-policy, which is an approximation of the optimal MDP strategy described above,  and use it
to define the $\delta$-regret.

We make the  following two assumptions. \\
\textit{Assumption 1:} The reward function $r(y)$ is $L_{r}$-lipschitz continuous, namely there exists a constant $L_{r}>0$ such that for all $x,y\in [0,1]$ and $x\leq y$, $r(y)-r(x)\leq L_r (y-x).$\\
\textit{Assumption 2:} The distribution function $F$ is Lipschitz continuous, namely there exists two constants  $L_{h}\geq L_{c}>0$ such that for all $x,y\in[0,1]$ and $x\leq y$, $L_{c}(y-x)\leq F(y)-F(x)\leq L_{h}(y-x).$

Assumption 1 is justified in practice,  as a small change in an action cannot increase the rewards indefinitely. 
In Assumption 2, the upper bound avoids the concentration of the threshold probability   over a small region. The lower bound avoids   intervals with zero probability of threshold. This assumption is  common in 
bandits with continuous action space \cite{yu2011unimodal}.

\begin{definition}
A $\delta$-policy of a MDP  performs action $\ell$ if the residual uncertainity  in state $(\ell,u,B_r)$ is $(u-\ell)\leq \delta$,  thus avoiding the risk of crossing the threshold;  otherwise it selects the optimal action, according to the Bellman equations.
\end{definition}   
In our setting, the $\delta$-policy selects an action $y_{\delta}^*(\ell,u,B_{r})$ in state $(\ell,u,B_r)$  such that
\begin{equation}
\label{eq:DeltaOptimal}
y_{\delta}^*(\ell,u,B_{r}) =\begin{cases}
        y^*(\ell,u,B_{r})  \qquad   \mbox{if } |u-\ell|>\delta, \\
         \ell \qquad \qquad \qquad  \mbox{  if } |u-\ell|\leq  \delta.\\
\end{cases}
\end{equation}
Here, $V_{\delta}^{*}(\ell,u,B_r)$ is the expected reward of the $\delta$-policy, given by (\ref{eq:DeltaOptimal}), in state $(\ell,u,B_r)$. 
The expected difference between the rewards of the optimal policy $V^{*}(0,1,B)$ and the rewards of the $\delta$-policy $V_{\delta}^{*}(0,1,B)$ is at most $\delta^2 B L_r L_h/(1-\gamma)$. 

There are two motivations for studying $\delta$-policies in our bandit setup. First, the risk of crossing the user's threshold is higher than the potential gains when $\delta$ is small. Since this risk cannot be quantified when $F$ is unknown, the platform can safely opt for a conservative policy when the residual uncertainty $u-\ell$ is small. 
Second, since the distribution $F$ is unknown, the platform needs to learn the conditional probabilities in (\ref{eq:conditionProb1}) in order to solve the MDP. The cost of learning these probabilities for a small UI $[\ell,u]$ is higher than the potential gains. To conclude, since in a continuous setting it is not possible to reduce the residual uncertainty to 0, the $\delta$-policies take into account that the risk of abandonment and cost of learning conditional probabilities  can outweigh the gains of reducing this uncertainty when the residual uncertainty is small.  

In the literature, a similar idea is used for designing   navigation schemes for robots using RL. In this case, robots are rewarded (i.e. the regret is zero) as they reach a distance of at most $\delta$ from the target  (see \cite{dhiman2018floyd} and references therein). In our setting, unlike the target point, the threshold $\theta_{n}$ is not known, and the parameter $\delta$ is defined with respect to the range of UI or the state space of the MDP. 
We now define the $\delta$-regret using the notion of $\delta$-policy.
\begin{definition} The $\delta$-regret of learning algorithm $\mathcal{A}$ is the expected difference between the total rewards received by the $\delta$-policy and the total rewards received by $\mathcal{A}$ from $N$ users, namely 
\begin{equation}\label{eq:DeltaRegret}
\begin{split}
        R_{\delta}(N)&= N V^*_{\delta}(0,1,B)-
    \mathbf{E}\bigg[ \sum_{n=1}^{N}\sum_{t=1}^{T_{n}}\gamma^{t-1} r(y_{n}(t))\mathbf{1}(y_n(t)\leq \theta_{n})\bigg].
\end{split}
\end{equation}
\end{definition}

\section{Upper Confidence Bound based Personalized Value Iteration}
We now  present the two variants of our algorithm Upper Confidence Bound based Personalized Value Iteration (UCB-PVI). These are UCB-PVI-SF for the \emph{soft feedback} model and 
UCB-PVI-HF for the \emph{hard feedback} model. 
We analyze their performance and establish their order optimality. 
For brevity, we provide a detailed presentation of  UCB-PVI-SF, and highlight the key changes to obtain UCB-PVI-HF. 

In UCB-PVI-SF, the platform divides the $N$ users into the exploration set $\mathcal{L}$ and the exploitation set $\mathcal{E}$. For users in $\mathcal{L}$, the platform performs an algorithm independent of the distribution $F$, the feedback probabilities $p_1$ and $p_2$, and the residual budget $B_r$. Unlike the classical MAB setting, since thresholds are never revealed, a noisy estimates of the thresholds  of the these users in $\mathcal{L}$ are utilized to estimate $F$, $p_1$ and $p_2$, and the measure of the noise in these estimates is dependent on the algorithm.  For users in $\mathcal{E}$, the platform performs an algorithm which uses the estimates of $F$, $p_1$ and $p_2$, and an estimation strategy for the un-observable budget $B_r$. 


\textit{Personalized policy for the Exploration set}\\
The platform uses   Algorithm \ref{alg:MainAlg} for each user $n\in \mathcal{L}$, which requires no   knowledge of the distribution $F$, positive feedback probability $p_1$, negative feedback probability $p_2$ and residual budget $B_r$. In this algorithm, at each iteration $j$, the platform performs a Linear Search Exploration (LSE) (see Algorithm \ref{alg:AdversaryStrategy}) on the  UI $[\ell,u]$ if the residual uncertainty $u-\ell$ is greater than an input parameter $\beta\in (0,1)$. 
If $u-\ell\leq\beta$, the platform chooses the conservative action $\ell$. By adjusting parameter $\beta$, the platform outweighs the risk of crossing the threshold in comparison to the gain from reducing the UI.

In Algorithm \ref{alg:AdversaryStrategy}, LSE reduces the expected residual uncertainty (or the range of UI)  by a factor of $O(1/\phi)$. 
This algorithm receives an UI $[\ell,u]$ as its input. It designs an action set $A$ in which the actions are chosen at an interval of length $I=(u-\ell)/\phi$ in the input UI. These actions are performed sequentially starting from a conservative action $\ell$ to an optimistic action $u$ in $A$. If a positive user's feedback is received, then the lower bound is updated to the current action. If a negative user's feedback is received, then the upper bound is updated to the current action, and the remaining actions in the set $A$ will not performed, since performing these action will reduce the patience of the user and will not reduce the range of UI any further.    In this algorithm, the parameter $d$ keeps an account of the discount factor based on the number of interactions between the platform and the user. Algorithm \ref{alg:AdversaryStrategy} reduces to an analogue of noisy binary search for $\phi=2$. Hence, we provide a more general analysis than this special case. 
Additionally, each user $n\in \mathcal{L}$ has its own personalized action depending on the feedback signal  at each step in LSE. The following theorem provides the regret bound 
of the algorithm for the set $\mathcal{L}$. 

\begin{theorem} \label{thm:exploration} Let $\tilde{\phi}=\phi/(\phi-1)$. 
For all $0\leq\delta<1$, $B$ such that $(\log_{\tilde{\phi}}(1/{\beta})+1)/(1-(p_1+p_2)(1-\min\{p_1,p_2\})^{\phi})\leq B$ and $p_1+p_2<1$, 
 the $\delta$-regret of the platform  using Algorithm 1 over the  set $\mathcal{L}$ 
 is
\begin{equation}
\begin{split}
  R_{\delta}(|\mathcal{L}|) 
  &\leq \frac{|\mathcal{L}|\log_{\tilde{\phi}}(1/\beta)}{1-(p_1+p_2)(1-\min\{p_1,p_2\})^{\phi}}\bigg(\frac{1-(1-p_2)^{\phi+1}}{p_2}\\
  &\qquad+L_r(\phi+1)+\frac{\beta L_r}{1-\gamma}\bigg)+|\mathcal{L}|(L_{r}(\phi+1)B+B).
\end{split}
\end{equation}
where $|.|$ is the cardinality of the set.
\end{theorem}
\begin{proof}
The proof of the theorem is presented in Appendix \ref{sec:thm1}
\end{proof}
\small
\begin{algorithm}[t]
\begin{algorithmic}
\STATE Input:  $\ell=0,u=1$, $\beta$, $j=0$, $\gamma$, $d=1$
\WHILE{user $n$ has not abandoned}
\IF{$u-\ell >\beta$}
\STATE $(\ell,u,d)=\mbox{LSE}(\ell,u,d,\gamma)$
\STATE $j=j+1$
\ELSE
\STATE Perform $\ell$.
\ENDIF
\ENDWHILE 
\STATE Output: Sum of the discounted rewards.
\caption{Personalized algorithm for all $n\in\mathcal{L}$}
\label{alg:MainAlg}
\end{algorithmic}
\end{algorithm}

\begin{algorithm}[t]
\begin{algorithmic}
\STATE Input:  $(\ell,u,d,\gamma)$;\\
\STATE Action set $A=\{\ell,\ell+I,\ldots,\ell+(\phi-1)I,u\}$, where $I=(u-\ell)/\phi$. 
\FOR{$a\in A$}
\STATE Perform $a$. 
\STATE Receive $d\cdot r(a)\mathbf{1}(a\leq\theta_{n})$, and set $d=\gamma\cdot d$
\IF{a positive feedback is available} 
    \STATE Update $\ell=a$.
\ENDIF 
\IF{a negative feedback is available} 
    \STATE Update $u=a$. Break;
\ENDIF

\ENDFOR
\STATE Output: $(\ell,u,d)$.
\caption{Linear search Exploration (LSE) on user $n$}
\label{alg:AdversaryStrategy}
\end{algorithmic}
\end{algorithm}
\normalsize
In Theorem \ref{thm:exploration}, the $\delta$-regret over the set $\mathcal{L}$ scales linearly with $|\mathcal{L}|$ and $B$. 
Thus, $|\mathcal{L}|$ should be carefully chosen to minimize the cumulative regret over all the $N$ users, which can be in the order of millions in many practical systems. 
Additionally, the regret has an $O(\log{1/{\beta}})$ dependency on the residual uncertainty $\beta$. 
Note that $\beta$ seems to play same role as $\delta$ in $\delta$-policy, however, it can be selected independent of $\delta$. 
As smaller uncertainty $\beta$ is desirable, the lower bound on $\beta$ is $\Omega(\tilde{\phi}^{-B})$ according to the assumption in Theorem \ref{thm:exploration}. 

The policy of each user is independent of the others in the set $\mathcal{L}$, therefore the platform can parallelize the design of the personalized policy for all these users.  When each user $n\in\mathcal{L}$ has either abandoned the platform or reached the steady state where $u-\ell\leq\beta$ in the Algorithm \ref{alg:MainAlg}, the platform utilizes the information gained from these users to estimate $F$, $p_1$ and $p_2$. Since the threshold $\theta_n$ is never revealed, if the user $n\in \mathcal{L}$ did not abandon the platform i.e. $u-\ell\leq\beta$, the noisy observation in the form of UI $[\ell,u]$ is available to estimate  $\theta_n$, $F$, $p_1$ and $p_2$. However, if the user has abandoned the platform, then there is no gain in reliable information about the threshold from this user. Unlike the classical MAB setting where each action returns an observation from an underlying distribution which can be utilized in the next step, in our setting, there may be users which do not provide any reliable information about their thresholds, and the count of these users is dependent on the feedback probabilities $p_1$ and $p_2$, and the platform's strategy, namely Algorithm \ref{alg:MainAlg}. 

Now, let $K$ be a random variable denoting the number of users in $\mathcal{L}$ that did not abandon the platform, namely
\begin{equation}
    K=\sum_{n\in \mathcal{L}}\textbf{1}(u_n-\ell_n\leq \beta),
\end{equation}
where $[\ell_n,u_n]$ is the final UI of user $n\in \mathcal{L}$. For all $x\in[0,1]$, the empirical estimate $\Hat{F}^{K}(x)$ of $F(x)$ is 
\begin{equation}\label{eq:empEstimate}
    \hat{F}^{K}(x)=\sum_{n\in \mathcal{L}}{\textbf{1}(u_n-\ell_n\leq \beta)\textbf{1}(\ell_n\leq x)}/{K}.
\end{equation}
The empirical estimate $\hat{p}_1$ of $p_1$ in UCB-PVI-SF is 
\begin{equation}\label{eq:PSFp}
   \hat{p}_1=\frac{\sum_{n\in \mathcal{L}}\sum_{t=1}^{T(n)}\textbf{1}(u_n-\ell_n\leq \beta)\mathbf{1}(S(y_n(t))=1)}{\sum_{n\in \mathcal{L}}\sum_{t^\prime=1}^{T(n)}\textbf{1}(u_{n^\prime}-\ell_{n^\prime}\leq \beta)\mathbf{1}(y_{n^\prime}(t^\prime)\leq \ell_{n^\prime}) },
\end{equation}
where $T(n)$ is the number of interaction between user $n \in \mathcal{L}$ and Algorithm \ref{alg:MainAlg} before $u_n-\ell_n\leq\beta$, and $S(y_{n}(t))$ denotes the feedback received from the user $n$ for the action $y_{n}(t)$. Hence,  $S(y_{n}(t))=1$ denotes that a positive feedback is received.

Similar to \eqref{eq:PSFp}, the empirical estimate $\hat{p}_2$ of $p_2$ in UCB-PVI-SF is 
\begin{equation}
    \hat{p}_2=\frac{\sum_{n\in \mathcal{L}}\sum_{t=1}^{T(n)}\textbf{1}(u_n-\ell_n\leq \beta)\mathbf{1}(S(y_{n}(t))=0)}{\sum_{n\in \mathcal{L}}\sum_{t^\prime=1}^{T(n)}\textbf{1}(u_{n^\prime}-\ell_{n^\prime}\leq \beta)\mathbf{1}(u_{n^{\prime}}\leq y_{n^\prime}(t^\prime))}.
\end{equation}


For UCB-PVI-HF, Algorithm \ref{alg:MainAlg} for the set $\mathcal{L}$ remains the same. In LSE,  either the lower bound or the upper bound are updated at every action $a\in A$ based on the feedback.
Similar to Theorem \ref{thm:exploration}, the $\delta$-regret over exploration set $\mathcal{L}$ of UCB-PVI-HF is also $O(|\mathcal{L}|B)$, presented in Theorem \ref{thm:explorationHF}. 

\textit{Personalized policy for the Exploitation set}\\
Unlike the exploration set, personalized policies for the users in the exploitation set are based on the estimates of $F$, $p_1$, $p_2$ and an estimation scheme of the residual budget $B_r$.

Using the estimate $ \hat{F}^{K}(x)$, the Upper Confidence Bound (UCB) and Lower Confidence Bound (LCB) of the conditional probability $P(A_{1}|\ell\leq \theta_{n}\leq u)=P(y|\ell,u)$ for the $\delta$-policy are 
\begin{equation}\label{eq:UCB}\begin{split}
   P_{U}^{K}(y|\ell,u)&=\frac{\hat{F}^{K}(u)-\hat{F}^{K}(y)}{\max\{\hat{F}^{K}(u)-\hat{F}^{K}(\ell),L_c(u-\ell)\}}
 +\frac{2(\eta_{K}+2\beta L_h)}{L_{c}\delta},
\end{split}\end{equation}
\begin{equation}\label{eq:LCB}\begin{split}
 P_{L}^{K}(y|\ell,u)&=\frac{\hat{F}^{K}(u)-\hat{F}^{K}(y)}{\max\{\hat{F}^{K}(u)-\hat{F}^{K}(\ell),L_c(u-\ell)\}}
 -\frac{2(\eta_{K}+2\beta L_h)}{L_{c}\delta},
\end{split}\end{equation}
  where $\eta_{K}=\sqrt{{18\log(16/\epsilon)}/{K}}$ and $\epsilon$ is a design parameter in $(0,1)$. 
 Likewise, the UCB and LCB of $P(A_{2}|\ell\leq \theta_{n}\leq u)=1-P(y|\ell,u)$ for the $\delta$-policy are
\begin{equation}\begin{split}
    &P^{K}_{U}(A_{2}|\ell\leq \theta_{n}\leq u)=1- P_{L}^{K}(y|\ell,u), \\  &P^{K}_{L}(A_{2}|\ell\leq \theta_{n}\leq u)=1- P_{U}^{K}(y|\ell,u).
    \end{split}
\end{equation}

 In UCB-PVI-SF, the platform uses the following strategy to design the personalized policies for the users in the set $\mathcal{E}$. For all $n\in\mathcal{E}$, the algorithm maintains a UI based on the feedback and an estimate $\hat{B}_r$ of remaining budget $B_r$, since $B_r$ is unknown to the platform. At start of interaction $t=0$ with user $n$, UI is $[0,1]$ and $\hat{B}_r=B$. At round $t$ of interaction, if the range of the UI $[\ell,u]$  is less than $\delta$, then the platform chooses the conservative action $\ell$ since the objective is to minimize $\delta$-regret. If the range of the UI $[\ell,u]$ is greater than $\delta$, then similar to (\ref{eq:negativeSoft}) and (\ref{eq:SoftUpper}), the action $\hat{y}(\ell,u,\hat B_{r})$ of the platform in the state $(\ell,u,\hat{B}_r)$ and its estimate of the maximum expected reward $ \hat{V}^*(\ell,u,\hat B_{r})$ are 
\begin{equation}\begin{split}\label{eq:UCBoptimalAction}
    &\hat{y}(\ell,u,\hat B_{r})=\mbox{argmax}_{\ell\leq y\leq u} \hat{V}^{K}_{U}(\ell,u,\hat B_{r},y), \\
    & \hat{V}^*(\ell,u,\hat B_{r})=\mbox{max}_{\ell\leq y\leq u} \hat{V}^{K}_{U}(\ell,u,\hat B_{r},y),
\end{split}
\end{equation}
 where $\hat{V}_{U}^{K}(\ell, u, \hat B_{r},y)$ is 
 the estimate of maximum expected rewards if action $y$ is performed in the state $(\ell, u,\hat B_{r})$, namely
 \begin{equation}\label{eq:UCBexpectedReward}\begin{split}
       \hat{V}_{U}^{K}(\ell,u,\hat B_{r},y) &= P_{U}^{K}(A_{1}|\ell,u)\big(r(y)+\gamma(1-\hat{p}_1)\hat{V}^*(\ell,u,\hat B_{r})\\
       &+\gamma\hat{p}_1\hat{V}^*(y,u,\hat B_{r}) \big)
       + P_{L}^{K}(A_{2}|\ell,u)\gamma\big(\hat{p}_2\hat{V}^*(\ell, y,\hat B_{r}-1)\\
       &+(1-\hat{p}_2)\hat{V}^*(\ell,u,\hat B_{r}-1) \big).\\
\end{split}
\end{equation}
Since $B_r$ of the user is un-observable to the platform, the following strategy is used to estimate $B_r$ at each interaction.  In UCB-PVI-SF, if a positive feedback is received from the user $n$, then 
$\ell=\hat{y}(\ell,u,\hat B_{r})$ and $\hat B_r=\hat B_r$. If a negative feedback is  received,  then $u=\hat{y}(\ell,u,\hat B_{r})$ and $\hat B_r=\hat B_r-1$. If the feedback is not received, then  $\ell$ and $u$ remain the same and 
  \begin{equation}\label{eq:BrEstimate}
    \hat B_r =
    \begin{cases*}
      \hat B_r & w.p.  $\frac{P^{K}_{U}(A_1|\ell,u)(1-\hat{p}_1)}{P^{K}_{U}(A_1|\ell,u)(1-\hat{p}_1)+P^{K}_{L}(A_2|\ell,u)(1-\hat{p}_2)}$ \\
      \hat B_r-1        &  w.p. $\frac{ P^{K}_{L}(A_2|\ell,u)(1-\hat{p}_2)}{P^{K}_{U}(A_1|\ell,u)(1-\hat{p}_1)+P^{K}_{L}(A_2|\ell,u)(1-\hat{p}_2)}$.
    \end{cases*}
  \end{equation}

The $\hat{V}_{U}^{K}(\ell,u,\hat B_{r},y)$ is an optimistic value (or UCB) of the expected reward in the state $(\ell,u,\hat B_{r})$ by performing $y$. 
Since $\ell\leq y\leq u$, we have that $V^{*}(\ell,u,\hat B_{r})\geq V^*(\ell,y,\hat B_r-1)$ because the UI $[\ell,u]$ has higher expected reward in comparison to the UI $[\ell,y]$, and the higher patience budget $\hat B_{r}$ allows a longer user-platform interaction.  This also implies $V^*(\ell,u,\hat{B}_r)\geq V^*(\ell,u,\hat{B}_r-1)$. Likewise, we have that $V^{*}(y,u,\hat B_{r})\geq V^*(\ell,u,\hat B_r-1)$ because the UI $[y,u]$ has a higher expected reward in comparison to the UI $[\ell,u]$, and the higher patience budget $\hat B_{r}$ allows a longer user-platform interaction. This also implies  $V^{*}(y,u,\hat B_{r})\geq V^*(\ell,y,\hat B_r-1)$. Thus, the event $A_{1}$ has higher expected rewards in comparison to $A_{2}$, and the sum of the conditional probabilities of these events is unity.
Therefore, $P_{U}^{K}(A_{1}|\ell,u)$ and $P_{L}^{K}(A_{2}|\ell,u)$ are used to compute an optimistic estimate $\hat{V}_{U}^{K}(\ell,u,\hat B_{r},y)$.  The following theorem provides the regret bound of the algorithm UCB-PVI-SF for
the set $\mathcal{E}$. 
\begin{theorem}\label{thm:exploitation} Let
\begin{equation}
   \Delta=\frac{(1-(1-p_2)^{\phi+1})\log_{\tilde{\phi}}(1/\beta)}{B p_2(1-(p_1+p_2)(1-\min\{p_1,p_2\})^{\phi})},
\end{equation}
$\tilde\lambda = \sqrt{\log (1/\lambda)/(2\Delta^2|\mathcal{L}|)}$, 
 $c={(1-\Delta(1+\tilde\lambda))}/{(1-\log_{\phi}(1/\beta)/B)}$ and $\tilde{K}=c |\mathcal{L}|$. Let the assumptions in Theorem \ref{thm:exploration} hold. 
For $\tilde{K}\geq 162\log(16/\epsilon)\max\{1/(1-p_1-p_2)^2,B\}$ and $0<\lambda< 1$  such that {$\tilde{\lambda}<1$},
the $\delta$-regret of the platform in UCB-PVI-SF over the  set $\mathcal{E}$ is
\begin{equation}
\begin{split}
    R_{\delta}(|\mathcal{E}|)&\leq |\mathcal{E}| 
    \frac{8(\eta_{\tilde K}+2\beta L_h)}{L_{c}\delta (1-\gamma)^2}\bigg({1+\gamma}+{\gamma \eta_{\tilde{K}}} B + \sqrt{B}+\frac{12 B}{(1-p_1-p_2)^2}\frac{(\eta_{\tilde K}+2\beta L_h)}{L_c\cdot \delta } \bigg).
\end{split}
\end{equation}
with probability at least $1-3\lambda\epsilon$.
\end{theorem}
\begin{proof}
The proof of the theorem is presented in Appendix \ref{sec:thm2}.
\end{proof}
 Similar to $\mathcal{L}$, each user in $\mathcal{E}$ has its own personalized policy depending on the feedback signal, and this policy can be parallelized for all users in $\mathcal{E}$.
 
 For the hard feedback model, the strategy  in (\ref{eq:UCBoptimalAction}) and (\ref{eq:UCBexpectedReward}) can be adapted using  (\ref{eq:optimalAction}) and (\ref{eq:expectedReward}) for UCB-PVI-HF. 
In UCB-PVI-HF, the action $\hat{y}(\ell,u,\hat B_{r})$ of the platform in the state $(\ell,u,\hat{B}_r)$ and its estimate of the maximum expected reward $ \hat{V}^*(\ell,u,\hat B_{r})$ are  
\begin{equation}
    \hat{y}(\ell,u,B_{r})=\mbox{argmax}_{\ell\leq y\leq u} \hat{V}^{K}_{U}(\ell,u,B_{r},y),
\end{equation}
\begin{equation}
    \hat{V}^*(\ell,u,B_{r})=\mbox{max}_{\ell\leq y\leq u} \hat{V}^{K}_{U}(\ell,u,B_{r},y),
\end{equation}
where 
\begin{equation}\begin{split}
       \hat{V}_{U}^{K}(\ell,u,B_{r},y) &= P_{U}^{K}(A_{1}|\ell,u)\big(r(y)+\gamma\hat{V}^*(y,u,B_{r})\big)\\
       &+ P_{L}^{K}(A_{2}|\ell,u)\gamma\big(\hat{V}^*(\ell,y,B_{r}-1)\big).\\
\end{split}
\end{equation}
Since the feedback is received for each action,  $B_r$ is observable in the hard feedback model, namely $\hat{B}_r=B_r$. The following theorem show the $\delta$-regret over the set $\mathcal{E}$ of UCB-PVI-HF. 
\begin{theorem}\label{thm:HFexploitation} For all $B$ such that $\log_{\phi}(1/\beta)+1\leq B$ and $\phi>1$, the $\delta$-regret of the platform in UCB-PVI-HF over the set $\mathcal{E}$ is
\begin{equation}
\begin{split}
    R_{\delta}(\mathcal{E})&\leq |\mathcal{E}|\frac{8(\eta_{|\mathcal{L}|}+2\beta L_h)}{L_{c}\delta }\bigg(\frac{1+\gamma}{(1-\gamma)^2}\bigg), 
\end{split}
\end{equation}
with probability at least $1-\epsilon$. 
\end{theorem}
\begin{proof}
The proof of the theorem is presented in Appendix \ref{sec:thm3}.
\end{proof}
Let us briefly compare the results in Theorems \ref{thm:exploitation} and \ref{thm:HFexploitation}. In Theorem \ref{thm:exploitation}, there is an additional factor of $O(|\mathcal{E}|\eta_{\tilde{K}}(\sqrt{B}+B\eta_{\tilde{K}}))$. This factor quantifies the regret contributed by the estimation strategy of the residual budget. Since the budget is observable in UCB-PVI-HF, we do not have any additional regret due to  estimation of $B_r$. 

Combining the regret over the sets  $|\mathcal{L}|$ and $|\mathcal{E}|$
using Theorem \ref{thm:exploration} and \ref{thm:exploitation},  the $\delta$-regret of UCB-PVI-SF is given in the following theorem.
\begin{theorem}\label{thm:deltaRegret} For {$|\mathcal{L}|= (\log(16/\epsilon))^{1/2}N^{2/3}(1-\tilde{\lambda})^{1/2}/(c B^{1/3}) $} and $\beta = \eta_{\tilde K}/2 L_h$, there exists a constant $M$ such that  with probability at least $1-3\lambda\epsilon$, the $\delta$-Regret of UCB-PVI-SF is
\begin{equation}
\nonumber
\begin{split}
     R_{\delta}(N)&\leq M N^{2/3}B^{2/3}\log^{2/3}(N). 
\end{split}
\end{equation}
\end{theorem}
\begin{proof}
The proof of the theorem is presented in Appendix \ref{sec:thm4}.
\end{proof}
Thus, the $\delta$-regret of the algorithm for the soft feedback model is $\tilde{O}(N^{2/3})$.
Similar to Theorem \ref{thm:deltaRegret},  the $\delta$-regret of UCB-PVI-HF is $\tilde{O}(N^{2/3})$, presented in Theorem \ref{thm:deltaHF}, for the hard feedback model. 
The following theorem provides a matching lower bound on the $\delta$-regret.

\begin{theorem}\label{thm:lowerBound} Let $L=L_r+L_h$.
For all $B$, $0\leq \delta<1$ and $N\geq \max\{L,0.08L^{2/3}\}$, 
there exists 
a constant $\tilde{M}$ such that 
the $\delta$-regret for any learning algorithm $\mathcal{A}$ satisfies 
\begin{equation}
\nonumber
   \sup_{\mathcal{F}(L)} R_{\delta}(N)\geq \tilde{M}L^{1/3}N^{2/3},
\end{equation}
where $\mathcal{F}(L)\hspace{-2pt}=\hspace{-2pt}\{F,r:\mbox{ $F$ is $L_{h}$-lipshitz continuous, $r$ is}$   $\mbox{$L_r$-lipshitz}$  continuous, and  $ L=L_r+L_h\}.$
\end{theorem}
\begin{proof}
The proof of the theorem is presented in Appendix \ref{sec:thm5}.
\end{proof}
The above theorem establishes that  the two variants of UCB-PVI are order optimal in $N$, which is analogous to a time horizon in classical MAB setting. Since $\theta_n$ is fixed for each user $n$, the platform personalizes the policy for each user by utilizing $B$, however it does not learn about $F$ along this process. The platform gains information about $F$ as it interacts with more users. Hence, the order optimality in our setting is studied with respect to $N$. Additionally, if $N$ is unknown, or new users join the platform, then the doubling trick can be used in conjunction with UCB-PVI to personalize the users' experience  \cite{besson2018doubling}. Finally, our lower bound may not be tight in terms of the parameter $B$. This is left as a future problem to study.

\subsection{Merits of the algorithm}
In many practical systems, millions of users interact  with the platform simultaneously. It is therefore desirable for the platform to personalize the users' experience as soon as possible and reduce the waiting time. 
In our case, the users in the exploration set receive a personalized experience without any delay. The users in the exploitation set  wait until for each $n\in\mathcal{L}$, either the user $n$ abandon the platform or the range of UI $[\ell_n,u_n]$ is less than $\beta=O(N^{2/3}/ B^{1/3})$. This requires  $O(\min\{B,\log(1/\beta)\})$ rounds of interactions, which corresponds to  the waiting time of the users in the exploitation set to obtain a personalized experience. Thus,  our proposed algorithm is also practical, as the waiting time is $O(B)$ and does not scale with $N$. In contrast, the classical UCB-type   strategies, for example the one in \cite{schmit2018learning}, have $O(N)$ waiting time, which makes them  impractical for large systems. 

In our algorithm, the users are divided into two sets $\mathcal{L}$ and $\mathcal{E}$. The number of users in the set $\mathcal{L}$ is $\Theta(N^{2/3})$, which is sub-linear in $N$. This implies that number of users in the set $\mathcal{E}$ is $\Theta(N)$. Hence, only a small number of users, namely $o(N)$ or $\Theta(N^{2/3})$, is needed to gain information, and this information is later utilized to improve the experience for  large number   of  users, namely $\Theta(N)$.

We also highlight key differences with respect to explore and commit strategies in MAB\cite{garivier2016explore}. In the exploration phase of these strategies, a single action is selected at each time and its reward is revealed, which is then used to calculate the expected reward for each action. In UCB-PVI,  we design a scheme for the selection of a “sequence of actions” for each user utilizing the user’s feedback, while the user’s threshold is never revealed. The noise in the estimation of the users' threshold depends on the algorithm, whereas the noise is zero in the exploration phase, independent of the algorithm.  In the commit phase,  a best fixed action is selected for all the remaining time. In our setting, this would be analogous to selecting a best fixed “sequence of actions” for all remaining users. Instead, UCB-PVI selects a “sequence of actions” for each user based on the feedback, hence, personalizing the experience.  In other words, UCB-PVI selects the next action for each user based on his or her  previous feedbacks. Hence, personalization is achieved by leveraging the users’ preferences over time.

Finally, let us briefly discuss the importance of learning across users in this setting. Consider the case when the information from previous users is not utilized to adapt the experience of the new users. In this case, a fixed  strategy is used across all the users which utilizes the budget $B$. This implies that the regret in this scenario will be linear in $N$. However, the regret of our learning algorithm is $\tilde O(N^{2/3})$.  Hence, the utilization of the information from previous users, say about threshold in this case, to adapt the experience of new users is essential to obtain sub-linear regret.

\subsection{Key Technical Contributions}
Unlike the classical MAB setting and its variants, the distribution $F$ of the threshold is learned without ever revealing the values of the thresholds of the individual users. We learn a noisy version of these thresholds using the feedback of the users in $\mathcal{L}$ and associate an Uncertainty Interval (UI) to each of them. Then, we only consider those users whose UI is less than $\beta$ in the exploration set. The parameter $\beta$ regulates the trade-off between the number of users whose samples are collected and the estimation accuracy. This trade-off is explicitly characterized in Theorems  \ref{thm:exploration}  and \ref{thm:exploitation}. 

Since $B_r$ is un-observable to the platform, we propose an estimation scheme. The technical contributions associated to this scheme are as follows. First, we establish that we can achieve an order optimal performance by using an estimation scheme of an un-observable variable along with the MDP like solution in place of solving a partially observable MDP. Second, we model the evolution of the true and the estimated residual budgets as two random walks and bound the expectation of the absolute difference between the states of the two random walk (rather than the expected difference between the states). The regret corresponding to this  scheme is quantified by the factor
$O(|\mathcal{E}|\eta_{\tilde{K}}(\sqrt{B}+\eta_{\tilde{K}}B))$ in  Theorem \ref{thm:exploitation}. 

The regret for the set $\mathcal{E}$ is due to two reasons: the unknown distribution $F$ and the unobservable residual budget $B_r$. Let the error introduced by the unknown $F$ be $e_1$ and the error introduced by the unobservable $B_r$ be $e_2$. First, we obtain concentration bounds on these errors. Now, the usual MAB trick would be to use the triangle inequality and bound the regret by $e_1+e_2$. This, however, leads to a linear regret bound $ O(N)$ in our setting. Instead, we adopt a novel approach, and bound the regret as $O(e_1\cdot e_2)$ (instead of $e_1+e_2$), leading to a sub-linear regret. 

\subsection{Numerical results}
In this section, we quantify the gains from utilizing users' feedback and personalizing their experiences. We evaluate the hard feedback model because this model incorporates maximum amount of users' feedback, and therefore it can quantify the maximum gains due to the utilization of the feedback. Additionally, we compare our algorithm with the Sequential Learning (SL) algorithm,  proposed by \cite{schmit2018learning}. $SL$ does not utilize users' feedback. 

In our experimental setup,
the reward function $r(y)=5y$, the distribution function $F$ is uniform $U[0,1]$, discount factor $\gamma=0.95$, $\delta=0.01$ and $\phi=2$. The performance of the algorithms are evaluated for varying  patience budget  $B$ and number of users $N$. The figures show expected total rewards along with the error bars representing one standard deviation from 5000 runs.  The experiments could not be conducted on real data sets due to the unavailability of such public datasets. Additionally, the literature \cite{schmit2018learning,LuCustomer2017,cao2019dynamic} on “Sequential Choice Model” also performs experiments in a simulated setting.

Figure \ref{fig:UCBPVI} shows the performance of our algorithm UCB-PVI-HF.  The figure  shows that the total reward  increases with the number of users $N$ as it increases the effective number of user-platform interactions. Additionally, the figure shows that the total reward increases as the budget increases. This is due to the fact that the higher budget $B$ allows longer user-platform interactions. Thus, more patient users can lead to higher rewards for the platform and in return, they can get a better personalized experience as the platform's actions would be closer to their threshold. This also leads to a conjecture that the lower bound on the regret is also a function of $B$, which is a future direction of research question. 

Figure \ref{fig:UCBPVIvsSL} shows  the performance of UCB-PVI-HF and SL. SL, proposed in  \cite{schmit2018learning}, is a standard UCB algorithm where the platform interacts with one user at a time. For each user-platform interaction, the action is selected based on  the maximum UCB on the expected rewards for a discretized action space. Since the algorithm was proposed in a setup where users have no patience, i.e $B<1$, it does not utilize the intermediate feedback from the users. 
Figure \ref{fig:UCBPVIvsSL} shows that UCB-PVI-HF outperforms SL for different values of  $B$ and  $N$. In Figure  \ref{fig:UCBPVIvsSL}, the total reward of both the algorithms increases with the budget $B$. For budget $ B<1$, 
the performance of both the  algorithms is close to each other as both the algorithms are order optimal for this special case \cite{schmit2018learning}. In contrast, as $B$ increases, the gap between the total rewards of the two algorithms increases, and UCB-PVI-HF outperforms SL by a greater margin. This is so because when $B$ is large, the users' feedback is  utilized by UCB-PVI-HF to personalize the future actions. On contrary, SL does not utilize this feedback from the users. This case is a further evidence of the gains from the utilization of the users' feedback in UCB-PVI-HF in comparison to SL, and need of personalization on various platforms like recommendation systems.

\begin{figure}
       \centering
\begin{subfigure}{}
\includegraphics[width=2.5 in]{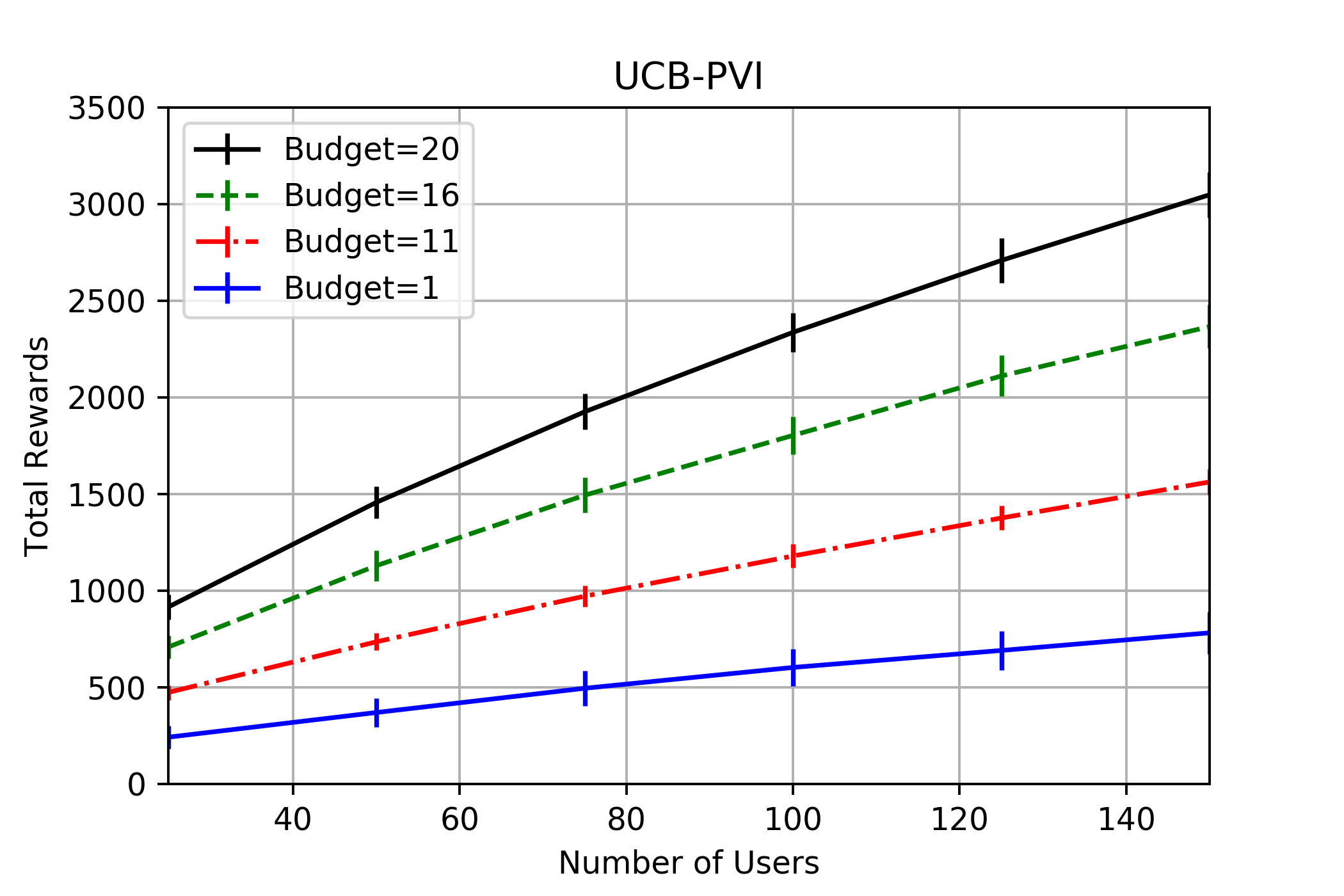}
\end{subfigure}
\caption{Performance evaluation of UCB-PVI-HF}
\label{fig:UCBPVI}
\end{figure}
\begin{figure}
       \centering
\begin{subfigure}{}
\includegraphics[width=4.5 in]{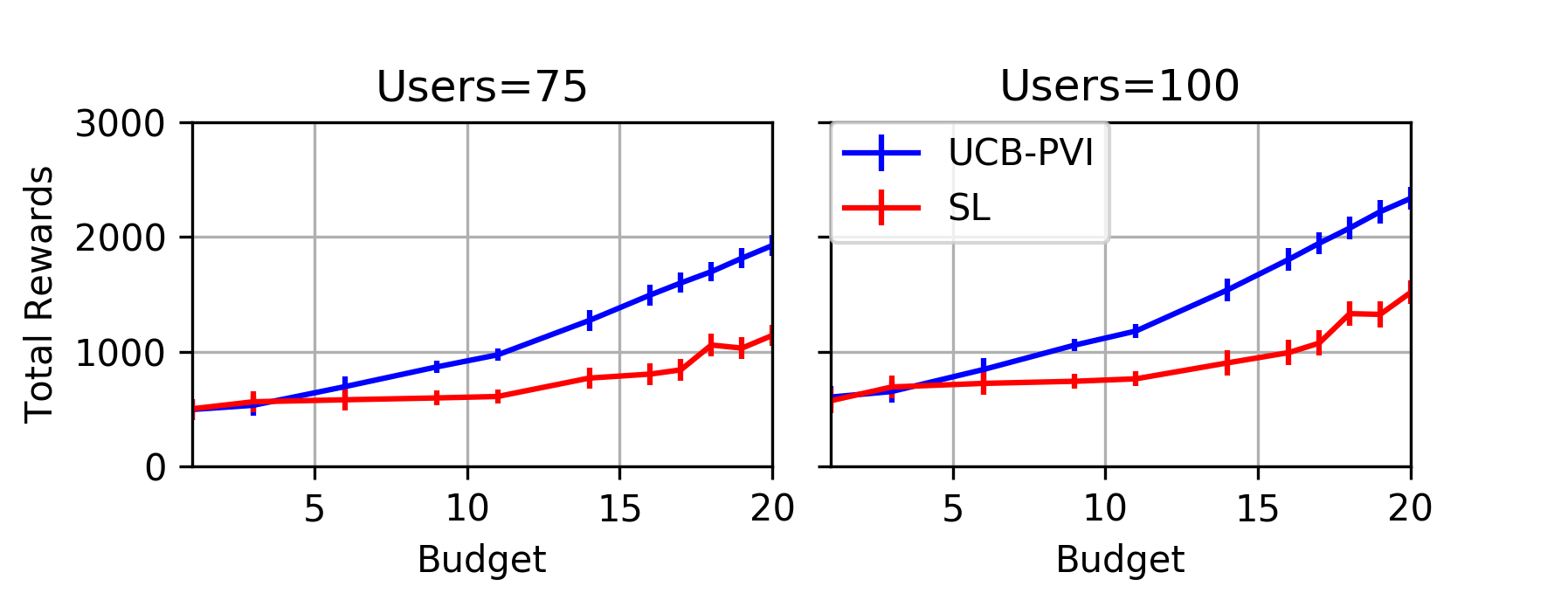}
\end{subfigure}
\begin{subfigure}{}
\includegraphics[width=4.5 in]{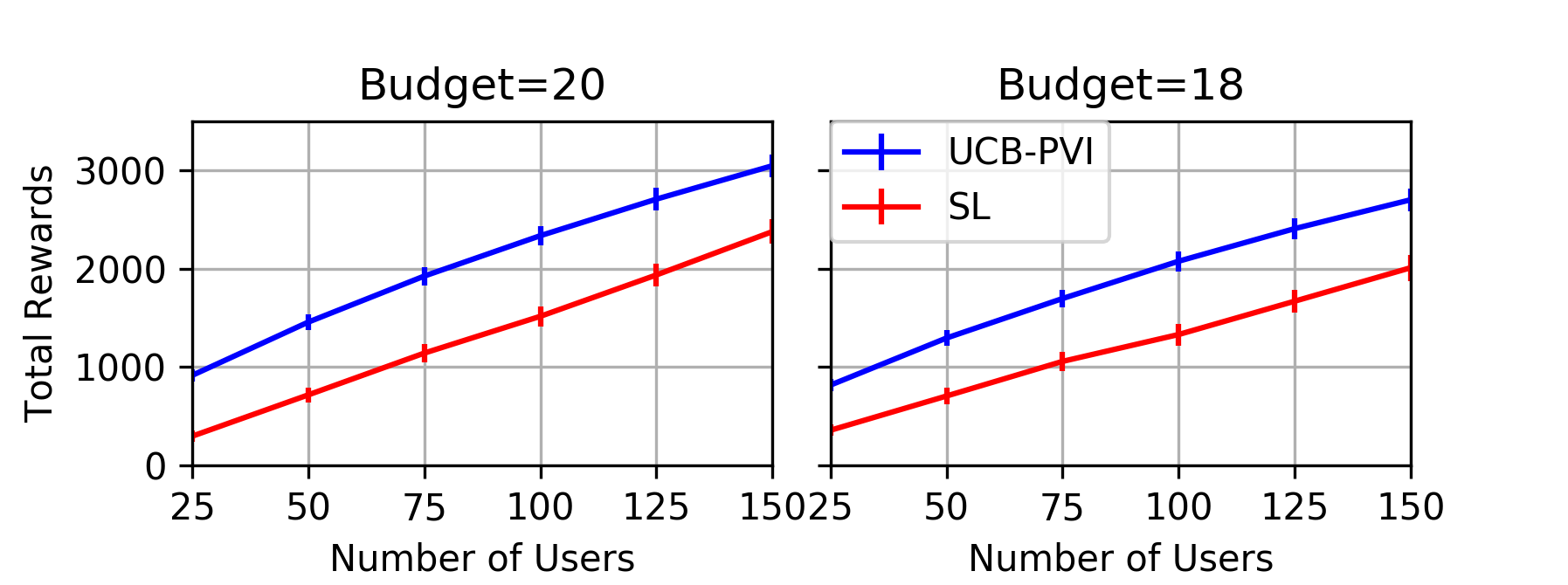}
\end{subfigure}
\caption{Comparison of UCB-PVI-HF and SL}
\label{fig:UCBPVIvsSL}
\end{figure}

\section{Conclusion}
We studied a novel  Bandit setting where a platform personalizes users' experience to maximize its rewards and the platform-user interaction is based on a TLM.
We studied two different  feedback models representing the information pattern available to the platform at each action, and propose two bandit algorithms for the feedback models, and showed that they are order optimal in $N$, up to a logarithmic factor. Also, our results immediately generalize to noisy thresholds, user-dependent initial budgets and patience depletion costs, and to mixed feedback models. 

The model studied here can be extended in various ways. First, the model can be extended to personalize users' experience if additional information about the user's characteristics is available to the platform. This context about the user can influence the user's threshold, the feedback model followed by the user and the patience budget $B$ of the user. This future direction corresponds to studying ``Contextual Sequential Choice Bandits". Second, similar to non-stochastic settings in \cite{auer2002nonstochastic,rangi2019online}, the learning problem can be studied when the users' thresholds are  assigned by an adversary, can be defined as ``Adversarial Sequential Choice Bandits". Finally, one could also study the TLM model in a social setting where the users interact with the platform in a group and adapt their thresholds depending on other users' experiences in their group. In conclusion, the model  studied in this work is a step towards utilizing the intermediate feedbacks from the users in recommendation systems, advertisement platforms and negotiations, and  can be extended in different directions depending on the specific nature of any platform.




\appendix
\section{Theorems and Proofs}
\subsection{Soft Feedback model}
\subsubsection{Proof of Theorem
\ref{thm:exploration}}\label{sec:thm1}

\begin{proof}
In Algorithm 1, letting $[\ell(j), u(j)]$ be the UI after  $j^{th}$ round of  Linear Search Exploration and $I_{j}=u(j)-\ell(j)$. 
Let $J$ be the first instance at which $u(J)-\ell(J)\leq \beta$, namely
\begin{equation}\label{eq:StopTimePNSF}
    J=\min\{j: u(j)-\ell(j)\leq \beta\}=\min\{j: I_j\leq \beta\}.
\end{equation}
{Since $K$ is the number of users in $\mathcal{L}$ that do not abandon the platform, we have $J<\infty$ with probability one for these users.}
The $\delta$-regret for these users is 
\begin{equation}\label{eq:newRegret}
    R_{\delta}(K)\leq \sum_{n=1}^{K}\mathbf{E}\bigg[\sum_{t=1}^{\infty}( \gamma^{t-1}r(\theta_{n})-r(y_{n}(t))\bigg|\theta_{n}\bigg],
\end{equation}
because given $\theta_n$, $r(\theta_{n})$ is the maximum achievable reward for a user and $\gamma<1$.  

{Assuming $J<\infty$, namely there exists a $J$ such that $u(J)-\ell(J)\leq \beta$.} We divide the regret analysis into two regions: until round $J$ and after round $J$. 

At round $j$, the input UI is $[\ell(j-1),u(j-1)]$. Let $Z_1$ be a random variable denoting the number of actions below the threshold, and $Z_2$ be a random variable denoting the number of actions above the threshold.  Then, we have
\begin{equation}\label{eq:sumOfRV}
    Z_1+Z_2=\phi+1.
\end{equation}
Let 
\begin{equation}
    E_1=\{\mbox{Feedback is received for action $\ell(j-1)$ only }\},
\end{equation}
\begin{equation}
    E_2=\{\mbox{Feedback is received for action $u(j-1)$ only }  \}.
\end{equation}
Then, 
\begin{equation}
    P(E_1)=p_1(1-p_1)^{z_1-1}(1-p_2)^{z_2}, P(E_2|Z_1,Z_2)=(1-p_1)^{z_1}(1-p_2)^{z_2-1}p_2. 
\end{equation}
Then, the probability that reduction is uncertainty is at least $1/\phi$
\begin{equation}\label{eq:lowerBoundProb1}
\begin{split}
       P(I_j\leq (1-1/\phi)I_{j-1})&=1-P(E_1)-P(E_2),\\
       &=1-p_1(1-p_1)^{z_1-1}(1-p_2)^{z_2}-(1-p_1)^{z_1}(1-p_2)^{z_2-1}p_2,\\
       &\geq 1-(1-p^*)^{\phi}p_1-(1-p^*)^{\phi}p_2= 1-(p_1+p_2)(1-p^*)^{\phi},
\end{split}
\end{equation}
where $p^*=\min\{p_1,p_2\}$, and the last inequality follows from \eqref{eq:sumOfRV}. 

Now, let us define a random variable $X_j=I_j/I_{j-1}$. Using \eqref{eq:lowerBoundProb1} and the fact that $X_j$ is a discrete random variable, we have
\begin{equation}\label{eq:prob1}
    P(X_j\leq (1-1/\phi))\geq 1-(p_1+p_2)(1-p^*)^{\phi},
\end{equation}
\begin{equation}\label{eq:prob2}
      P(X_j=1)\leq (p_1+p_2)(1-p^*)^{\phi}.
\end{equation}
Using \eqref{eq:StopTimePNSF} and the fact that $I_0=1$, we have
\begin{equation}
\begin{split}
    J&=\min\{j: I_j\leq \beta\},\\
    &=\min\{j: \prod_{i=1}^{j}X_i\leq \beta\},\\
    &=\min\{j: \sum_{i=1}^{j}\log(X_i)\leq \log(\beta)\},\\
    &=\min\{j: \sum_{i=1}^{j}\log(1/X_i)\geq \log(1/\beta)\},\\
\end{split}
\end{equation}
where the last equality follows from the fact that $X_i\leq 1$ and $\beta\leq 1$. Using Wald's equation, we have
\begin{equation}\label{eq:checkJ}
   E[J]\leq \frac{\log(1/\beta)}{E[\log(1/X_i)]}+1. 
\end{equation}
Now, using \eqref{eq:prob1} and \eqref{eq:prob2}, we have 
\begin{equation}
    E[\log(1/X_i)]\geq (1-(p_1+p_2)(1-p^*)^{\phi})\log(\phi/(\phi-1)),
\end{equation}
which implies along with \eqref{eq:checkJ} that
\begin{equation}\label{eq:UppBoundJ}
     E[J]\leq \frac{\log(1/\beta)}{ (1-(p_1+p_2)(1-p^*)^{\phi})\log(\phi/(\phi-1))}=\frac{\log_{\tilde{\phi}}(1/\beta)}{1-(p_1+p_2)(1-p^*)^{\phi}},
\end{equation}
where $\tilde{\phi}=\phi/(\phi-1)$. 
Now, at $j^{th}$ round of LSE, let $Z_j$ be the random variable denoting the expected number of actions performed by the platform above the threshold. Since the platform halts LSE as soon as it receives a negative feedback, we have
\begin{equation}
    P(Z_j=k)=(1-p_2)^{k-1}p_2.
\end{equation}
Since $Z_j\leq \phi+1$, we have that
\begin{equation}\label{eq:ThresholdAbove}
\begin{split}
    E[Z_j]&\leq\sum_{k=1}^{\phi+1} k(1-p_2)^{k-1}p_2,\\
    &\stackrel{(a)}{\leq}\frac{1-(1-p_2)^{\phi+1}}{p_2},
\end{split}
\end{equation}
where $(a)$ follows from the fact that for $r<1$, we have  $\sum_{n=1}^{N}nr^{n}\leq r(1-r^N)/(1-r)^2$. Given $J<\infty$, combining \eqref{eq:UppBoundJ} and \eqref{eq:ThresholdAbove}, the expected number of actions above the threshold are at most
\begin{equation}\label{eq:UppBound}
    \frac{1-(1-p_2)^{\phi+1}}{p_2}\cdot \frac{\log_{\tilde{\phi}}(1/\beta)}{1-(p_1+p_2)(1-p^*)^{\phi}},
\end{equation}
which is less than $B$ from the assumption in the theorem. This is also the expected regret from actions above the threshold.

Now, at round $j$, the regret from actions below the threshold is at most $L_r(\phi+1)$. This implies the expected regret  from the actions below the threshold until round $J$ are at most
\begin{equation}\label{eq:belowThresholdRegret}
  \frac{L_r(\phi+1)\log_{\tilde{\phi}}(1/\beta)}{1-(p_1+p_2)(1-p^*)^{\phi}}. 
\end{equation}
Combining \eqref{eq:ThresholdAbove} and \eqref{eq:belowThresholdRegret}, we have that the expected regret until round $J$ is at most
\begin{equation}\label{eq:untillJ}
    \frac{\log_{\tilde{\phi}}(1/\beta)}{1-(p_1+p_2)(1-p^*)^{\phi}}\bigg(L_r(\phi+1)+\frac{1-(1-p_2)^{\phi+1}}{p_2}\bigg). 
\end{equation}
The expected regret after round $J$ is at most
\begin{equation}\label{eq:sf11}
    \beta L_r/(1-\gamma).
\end{equation}
Now, let us consider the users with $J=\infty$,  namely the user abandon before the residual uncertainity becomes less than $\beta$. Given $J=\infty$, the regret for this user is at most
\begin{equation}\label{eq:JinftyNSF}
    L_{r}(\phi+1)B+B.
\end{equation}
Combining \eqref{eq:untillJ} and \eqref{eq:JinftyNSF}, the regret is at most
\begin{equation}
    |\mathcal{L}|\frac{\log_{\tilde{\phi}}(1/\beta)}{1-(p_1+p_2)(1-p^*)^{\phi}}\bigg(L_r(\phi+1)+\frac{1-(1-p_2)^{\phi+1}}{p_2}+\frac{\beta L_r}{1-\gamma}\bigg) +|\mathcal{L}|(L_{r}(\phi+1)B+B).
\end{equation}

\end{proof}

%

\subsubsection{Proof of Theorem \ref{thm:exploitation}}\label{sec:thm2} 
\begin{proof} 
Let $G^*(B)=V^*_{\delta}(0,1,B)$ be the  expected reward  received by the $\delta$-policy of the MDP for patience budget $B$ for each user.  
Given the estimate $\hat{F}^{K}(x)$ and estimate $\hat{B_r}$ of remaining budget for users in $\mathcal{E}$, $\hat{G}(\hat{B_r})=\hat{V}^*(0,1,\hat{B_r})$ is the estimate of expected rewards for each user in $\mathcal{E}$.   $G(B_r)$ 
 is the actual expected rewards  of the platform's policy for each user and $B_r$ is the actual remaining budget. Also, at $t=0$, for UI $[0,1]$, $\hat{B_r}=B_r=B$.  In the notations, we drop the parameter $\delta$ for simplicity.

Thus, the cumulative regret over all users in $\mathcal{E}$ with respect to $\delta$-optimal policy is 
\begin{equation}\label{eq:MainCheck}
\begin{split}
        R_{\delta}(E) &= |\mathcal{E}|( G^*(B)-G(B)),\\
                &=|\mathcal{E}|(G^*(B)-\hat{G}(B)+\hat{G}(B)-G(B)).
\end{split}
\end{equation}
In the above equation, we bound the two terms $R^{+}(B)=G^*(B)-\hat{G}(B)$ and $R^{++}(B)=\hat{G}(B)-G(B)$ individually.

For $R^{+}(B)$, given $B_{r}\leq B$ and the estimate $\hat{B_r}$, we define 
\begin{equation}
\begin{split}
     E^{+}(B_{r},\hat{B_r})&= \sup_{0\leq \ell\leq u\leq 1} \big( {V}^*_{\delta}(\ell,u,B_{r})-\hat{V}^*(\ell,u,\hat{B_r})\big),
\end{split}
\end{equation}
and
\begin{equation}
\begin{split}
     E^{+}(B_{r})&= \sup_{0\leq\ell\leq u\leq 1} \big( {V}^*_{\delta}(\ell,u,B_{r})-\hat{V}^*(\ell,u,{B_r})\big).
\end{split}
\end{equation}
Given $B_r$, $E^{+}(B_{r})$ is the maximum error in estimation of value function.
For any $B_{r}\leq B$ and $x\geq 0$, the following holds 
\begin{equation}\label{eq:maximumErrorBound1.1.1}
\begin{split}
     E^{+}(B_{r})&=\sup_{0\leq \ell\leq u\leq 1} \big( {V}^*_{\delta}(\ell,u,B_{r})-\hat{V}^*(\ell,u,{B_r})\big),\\
     &\stackrel{(a)}{\leq} \sup_{0\leq\ell\leq u\leq 1} \bigg( {V}^*_{\delta}(\ell,u,B_{r}-x)-\hat{V}^*(\ell,u,{B_r}-x) + \frac{x}{1-\gamma}\bigg),\\
     &=E^{+}(B_{r}-x)+\frac{x}{1-\gamma},
\end{split}
\end{equation}
where $(a)$ follows from the following facts:  for all $B_{r}< 0$, $V^*_{\delta}(\ell,u,B_{r})=\hat{V}(\ell,u,B_r)=0$; $r(y_{n}(t))\leq 1$ and the expectation is performed with respect to conditional probabilities $P(y|\ell,u)$ and $P_{U}^{K}(y|\ell,u)$ which lies in $[0,1]$; the maximum reward due to  platform-user interactions in $x$ budget is $x/(1-\gamma)$.
 Also, using the fact that the rewards  function $r(y)$ is non-decreasing in $y$,  for $u_{1}\leq \ell_{2}$ and $B_{1}\leq B_{2}$, we have that
\begin{equation}\label{eq:boundValueFunction1.1.1}
    V^*_{\delta}(\ell_2,u_2,B_{2})-V^*_{\delta}(\ell_1,u_1,B_{1})\leq \frac{B_{2}-B_{1}}{1-\gamma}.
\end{equation}
Likewise, 
\begin{equation}\label{eq:boundValueFunction2.1}
    \hat{V}^*(\ell_2,u_2,B_{2})-\hat{V}^*(\ell_1,u_1,B_{1})\leq \frac{B_{2}-B_{1}}{1-\gamma}.
\end{equation}
Also, we have
\begin{equation}\label{eq:Budget1}
\begin{split}
     E^{+}(B_{r},\hat{B_r})&=\sup_{0\leq \ell\leq u\leq 1} \big( {V}^*_{\delta}(\ell,u,B_{r})-\hat{V}^*(\ell,u,\hat{B_r})\big),\\
     &=\sup_{0\leq\ell\leq u\leq 1} \big( {V}^*_{\delta}(\ell,u,B_{r})-\hat{V}^*(\ell,u,{B_r})+\hat{V}^*(\ell,u,{B_r})-\hat{V}^*(\ell,u,\hat{B_r})\big),\\
     &\stackrel{(a)}{\leq}E^{+}(B_{r})+\frac{|B_r-\hat{B}_r|}{1-\gamma},
\end{split}
\end{equation}
where $(a)$ follows from the fact that maximum reward due to platform-user interaction in one unit budget is $1/1-\gamma$. 

Now, let us first bound $|B_r-\hat{B}_r|$.  Given $B_r$, $\hat{B}_{r}$ and feedback is not received, we have
\[B_{r}=B_{r}-1 \mbox{ with probability } \frac{P(A_2|\ell,u)(1-p_2)}{P(A_1|\ell,u)(1-p_1)+P(A_2|\ell,u)(1-p_2)},\]
and 
\[B_{r}=B_{r} \mbox{ with probability } \frac{P(A_1|\ell,u)(1-p_1)}{P(A_1|\ell,u)(1-p_1)+P(A_2|\ell,u)(1-p_2)}.\]
Likewise, $\hat{B}_r$ evolves as 
\[\hat{B}_{r}=\hat{B}_{r}-1 \mbox{ with probability } \frac{P^{K}_{L}(A_2|\ell,u)(1-\hat{p}_2)}{P^{K}_{U}(A_1|\ell,u)(1-\hat{p}_1)+P^{K}_{L}(A_2|\ell,u)(1-\hat{p}_2)},\]
and 
\[\hat{B}_{r}=\hat{B}_{r} \mbox{ with probability } \frac{P^{K}_{U}(A_1|\ell,u)(1-\hat{p}_1)}{P^{K}_{U}(A_1|\ell,u)(1-\hat{p}_1)+P^{K}_{L}(A_2|\ell,u)(1-\hat{p}_2)}.\]
$B_r$ and $\hat{B}_r$ can be modeled as two random walk bounded between $[0,B]$. We also have
{\small\begin{equation}\label{eq:RWbound1}
\begin{split}
&\bigg|\frac{P(A_2|\ell,u)(1-p_2)}{P(A_1|\ell,u)(1-p_1)+P(A_2|\ell,u)(1-p_2)}-\frac{P^{K}_{L}(A_2|\ell,u)(1-\hat{p}_2)}{P^{K}_{U}(A_1|\ell,u)(1-\hat{p}_1)+P^{K}_{L}(A_2|\ell,u)(1-\hat{p}_2)}\bigg| \\
&\stackrel{(a)}{\leq} \bigg|\frac{P(A_2|\ell,u)(1-p_2)P^{K}_{U}(A_1|\ell,u)(1-\hat{p}_1)-P^{K}_{L}(A_2|\ell,u)(1-\hat{p}_2)P(A_1|\ell,u)(1-p_1) }{(1-p_1-p_2)(1-\hat{p}_1-\hat{p}_2)}\bigg|\\
&\stackrel{(b)}{\leq} \frac{|P(A_2|\ell,u)- P^{K}_{L}(A_2|\ell,u)|+|p_1-\hat{p}_1|+|p_2-\hat{p}_2|+|P^{K}_{U}(A_1|\ell,u)- 
P(A_1|\ell,u)|}{(1-p_1-p_2)(1-\hat{p}_1-\hat{p}_2)}\\
&\stackrel{(c)}{\leq} \frac{2}{(1-p_1-p_2)^2-(1-p_1-p_2)\eta_{\tilde{K}}}\bigg( 
\frac{(\eta_{\tilde K}+2\beta L_h)}{L_c\cdot \delta }+\eta_{\tilde{K}}
\bigg)\\
&\stackrel{(d)}{\leq} \frac{6}{(1-p_1-p_2)^2}\frac{(\eta_{\tilde K}+2\beta L_h)}{L_c\cdot \delta },
\end{split}
\end{equation}}
with probability at least $1-3\epsilon\lambda$, where $(a)$ follows from the fact that for $x,y\in(0,1)$, we have $1-xy\geq 1-x$ $(b)$ follows from the fact that the terms in the product are less than unity, $(c)$ follows from Theorems \ref{thm:UCBandLCB} and \ref{thm:feedbackEstimate}, and $(d)$ follows from the 
assumption $\eta_{\tilde{K}}\leq (1-p_1-p_2)/3$ in the theorem. Likewise, we have
{\small\begin{equation}\label{eq:RWbound2}
\begin{split}
    &\bigg|\frac{P^{K}_{U}(A_1|\ell,u)(1-\hat{p}_1)}{P^{K}_{U}(A_1|\ell,u)(1-\hat{p}_1)+P^{K}_{L}(A_2|\ell,u)(1-\hat{p}_2)}-\frac{P(A_1|\ell,u)(1-p_1)}{P(A_1|\ell,u)(1-p_1)+P(A_2|\ell,u)(1-p_2)}\bigg|\\
    &\leq \frac{6}{(1-p_1-p_2)^2}\frac{(\eta_{\tilde K}+2\beta L_h)}{L_c\cdot \delta },
\end{split}
\end{equation}}
with probability $1-3\lambda\epsilon$.

At the start of interaction between platform and user i.e. $t=0$, $\hat{B}_{r}=B_r=B$. For all $0\leq\ell\leq u\leq 1$, using \eqref{eq:RWbound1} and \eqref{eq:RWbound2}, the expected difference at next interaction $t=1$ is 
{{\small
\begin{equation}\label{eq:Bestimate}
\begin{split}
       &\mathbf{E}[B_r-\hat{B}_r]\\
       &\leq 2\bigg|\frac{P(A_2|\ell,u)(1-p_2)}{P(A_1|\ell,u)(1-p_1)+P(A_2|\ell,u)(1-p_2)}-\frac{P^{K}_{L}(A_2|\ell,u)(1-\hat{p}_2)}{P^{K}_{U}(A_1|\ell,u)(1-\hat{p}_1)+P^{K}_{L}(A_2|\ell,u)(1-\hat{p}_2)}\bigg|\\
       &\qquad +2\bigg|\frac{P^{K}_{U}(A_1|\ell,u)(1-\hat{p}_1)}{P^{K}_{U}(A_1|\ell,u)(1-\hat{p}_1)+P^{K}_{L}(A_2|\ell,u)(1-\hat{p}_2)}-\frac{P(A_1|\ell,u)(1-p_1)}{P(A_1|\ell,u)(1-p_1)+P(A_2|\ell,u)(1-p_2)}\bigg|,\\
       &\leq \frac{12}{(1-p_1-p_2)^2}\frac{(\eta_{\tilde K}+2\beta L_h)}{L_c\cdot \delta },
\end{split}
\end{equation}
}}
with probability at least $(1-3\lambda\epsilon)$.
Since the budget is $B$, for all $B_{r}$ and $\hat{B}_{r}$, the expected difference between the estimates is at most
\begin{equation}\label{eq:ExpectRandomWalk}
\begin{split}
    \mathbf{E}[B_r-\hat{B}_r]&\leq \frac{12 B}{(1-p_1-p_2)^2}\frac{(\eta_{\tilde K}+2\beta L_h)}{L_c\cdot \delta }.
\end{split}
\end{equation}
Combining \eqref{eq:ExpectRandomWalk} and the fact that the variance of random walk is $O(B)$, we have
\begin{equation}\label{eq:BudgetBound}
    \mathbf{E}[|B_r-\hat{B}_r|]\leq \sqrt{B}+\frac{12 B}{(1-p_1-p_2)^2}\frac{(\eta_{\tilde K}+2\beta L_h)}{L_c\cdot \delta }. 
\end{equation}

Now, we  bound $R^{+}(B)=G^*(B)-\hat{G}(B)$ in terms of $E^{+}(B_r,\hat{B_r})$. Also, $B=\hat{B_r}=B_r$ at $t=0$. Now, suppose
\begin{equation}\label{eq:Assu1}
    P_{U}^{K}(y|\ell,u)=P(y|\ell,u)+v(\ell,u,y),    \Hat{p}_1=p_1+c_1 \mbox{ and } \Hat{p}_2=p_2+c_2.
\end{equation}

For all $u\in [0,1]$ such that $u-\ell>\delta$, $B^\prime\leq B$, and estimate $\hat{B}^{\prime}=B^\prime-b$ the lower bound on $\hat{V}^*(u,B^\prime)$ is as follows:
{\small\begin{equation}
\begin{split}
    &\hat{V}^*(\ell,u,\hat{B}^{\prime})\\
    &\stackrel{(a)}{\geq}  \hat{V}^{K}_{U}(\ell,u,B^\prime-b,y^*_{\delta}(\ell,u,B^\prime)),\\
                 &= P_{U}^{K}\big(y^*_{\delta}(\ell,u,B^\prime)|\ell,u\big)\bigg(r({y}^*_{\delta}(\ell,u,B^\prime))+\gamma(1-\hat{p}_1)\hat{V}^*(\ell,u,B^\prime-b)+ \gamma\hat{p}_1 \hat{V}^*(y^*_{\delta}(\ell,u,B^\prime),u,B^\prime-b)\bigg),\\
                 &\qquad +\big(1-P_{U}^{K}(y^*_{\delta}(\ell,u,B^\prime)|\ell,u)\big)\gamma\bigg(\hat{p}_2\hat{V}^*(\ell,y^*_{\delta}(u,B^\prime),B^\prime-b-1)+(1-\hat{p}_2)\hat{V}^*(\ell,u,B^\prime-b-1)\bigg),\\
       &\stackrel{(b)}{=} \bigg(P\big({y}^*_{\delta}(\ell,u,B^\prime)|\ell,u\big)+v(\ell,u,y^*_{\delta}(\ell,u,B^\prime))\bigg)\big(r({y}^*_{\delta}(\ell,u,B^\prime))+\gamma (1-p_1-c_1)\hat{V}^*(\ell,u,B^\prime-b) \\
       &\qquad+\gamma (p_1+c_1)\hat{V}^*({y}^*_{\delta}(\ell,u,B^\prime),u,B^\prime-b)
       \big)\\
       &\quad+\bigg(1-P\big({y}^*_{\delta}(\ell,u,B^\prime)|\ell,u\big)-v(\ell,u,y^*_{\delta}(\ell,u,B^\prime))\bigg)\gamma\bigg(({p}_2+c_2)\hat{V}^*(\ell, y^*_{\delta}(\ell,u,B^\prime),B^\prime-b-1)\\
       &\qquad +(1-{p}_2-c_2)\hat{V}^*(\ell,u,B^\prime-b-1)\bigg),\\
       &\stackrel{(c)}{\geq}  \bigg(P\big({y}^*_{\delta}(\ell,u,B^\prime)|\ell,u\big)+v(\ell,u,y^*_{\delta}(\ell,u,B^\prime))\bigg)\bigg(r({y}_{\delta}^*(\ell,u,B^\prime))+\gamma(1-p_1-c_1){V}^*_{\delta}(\ell,u,B^\prime)\\
       &\qquad +\gamma (p_1+c_1) {V}^*_{\delta}({y}_{\delta}^*(\ell,u,B^\prime), u,B^\prime)
       -\gamma (1-p_1-c_1) E^{+}(B^\prime,\hat{B}^\prime)-\gamma (p_1+c_1) E^{+}(B^\prime,\hat{B}^\prime)\bigg)\\
       &\qquad+\bigg(1-P\big({y}_{\delta}^*(\ell,u,B^\prime)|\ell,u\big)-v(\ell,u,y_{\delta}^*(\ell,u,B^\prime))\bigg)\gamma\bigg(({p}_2+c_2){V}^*_{\delta}(\ell,y^*_{\delta}(\ell,u,B^\prime),B^{\prime}-1)\\
       &\qquad +(1-{p}_2-c_2){V}^*_{\delta}(\ell, u,B^\prime-1) - ({p}_2+c_2)E^{+}(B^\prime-1,\hat{B}^{\prime}-1))\\
       &\qquad-(1-{p}_2-c_2)E^{+}(B^\prime-1,\hat{B}^{\prime}-1)\bigg),\\
       &\stackrel{(d)}{\geq}  V^*_{\delta}(\ell,u,B^\prime) - \gamma E^{+}(B^\prime,\hat{B}^\prime)+ v(\ell,u,y^*_{\delta}(\ell,u,B^\prime))r(y^*_{\delta}(\ell,u,B^\prime))\\
       &\qquad-v(\ell,u,y^*_{\delta}(\ell,u,B^\prime))\gamma\big(E^{+}(B^\prime,\hat{B}^\prime)-E^{+}(B^\prime-1,\hat{B}^{\prime}-1)\big)\\
       &\qquad+ v(\ell,u,y^*_{\delta}(\ell,u,B^\prime))\gamma\bigg((1-p_1){V}^*_{\delta}(\ell,u,B^\prime)+p_1{V}^*_{\delta}(y^*_{\delta}(\ell,u,B^\prime),u,B^\prime)\\
       &\qquad -p_2{V}^*_{\delta}(\ell,y^*_{\delta}(\ell,u,B^\prime),B^\prime-1)-(1-p_2){V}^*_{\delta}(\ell,u,B^\prime-1)\bigg)\\
        &\qquad +v(\ell,u,y^*_{\delta}(\ell,u,B^\prime))\gamma c_1 \bigg({V}^*_{\delta}(y^*_{\delta}(\ell,u,B^\prime),u, B^\prime)-{V}^*_{\delta}(\ell,u,B^\prime)\bigg),\\
       &\qquad -v(\ell,u,y^*_{\delta}(\ell,u,B^\prime))\gamma c_2 \bigg({V}^*_{\delta}(\ell,y^*_{\delta}(\ell,u,B^\prime),B^\prime-1)-{V}^*_{\delta}(\ell,u,B^\prime-1)\bigg),\\
       &\stackrel{(e)}{\geq} V^*_{\delta}(\ell,u,B^\prime)-|v(\ell,u,y^*_{\delta}(\ell,u,B^\prime))|\bigg(1+2\frac{\gamma}{1-\gamma}\bigg)-|v(\ell,u,y^*_{\delta}(\ell,u,B^\prime))\frac{|B^{\prime}-\hat{B}^\prime|}{1-\gamma}\\
       &-\gamma|c_1\cdot v(\ell,u,y^*_{\delta}(\ell,u,B^\prime))|\frac{B^{\prime}}{1-\gamma}-\gamma|c_2\cdot v(\ell,u,y^*_{\delta}(\ell,u,B^\prime))|\frac{B^{\prime}}{1-\gamma}-\gamma E^{+}(B^\prime,\hat{B}^\prime),\\
\end{split}
\end{equation}}
where $(a)$ follows from the fact that the optimal action according to $\hat{V}^*(\ell,u,\hat{B}^\prime)$ is $\hat{y}(\ell,u,\hat{B}^\prime)$; $(b)$  follows from \eqref{eq:Assu1} 
; $(c)$ follows from the definition of $E^+(B^\prime,\hat{B}^\prime)$, $(d)$ follows from the fact that the probability lies in $[0,1]$ and $E^{+}(B^\prime,\hat{B}^\prime)\geq E^{+}(B^\prime-1,\hat{B}^\prime-1)$,because  the instance or the optimal policy that corresponds to value of $ E^{+}(B^\prime-1,\hat{B}^\prime-1)$ is also be  feasible in budget $B^\prime$ and $\hat{B}^\prime$; $(e)$ follows from (\ref{eq:maximumErrorBound1.1.1}), (\ref{eq:boundValueFunction1.1.1}),\eqref{eq:boundValueFunction2.1} and (\ref{eq:Budget1}). 
From the above equation,  we have that
{\small
\begin{equation}\begin{split}\label{eq:DiffBetwVFnEstimate}
   &V^*_{\delta}(\ell,u,B^\prime)- \hat{V}^*(\ell,u,\hat{B}^\prime)\\
   &\leq |v(\ell,u,y^*(\ell,u,B^\prime))|\bigg(\frac{1+\gamma}{1-\gamma}+ \frac{|B^{\prime}-\hat{B}^\prime|}{1-\gamma} 
   \bigg)+\gamma(|c_1|+|c_2|) |v(\ell,u,y^*_{\delta}(\ell,u,B^\prime))|\frac{B^{\prime}}{1-\gamma}+\gamma E^{+}(B^\prime,\hat{B}^\prime)\\
   & \stackrel{(a)}{\leq}  |P^{K}_{U}(y|\ell,u)-P(y|\ell,u)|\bigg(\frac{1+\gamma}{1-\gamma}+\frac{|B^{\prime}-\hat{B}^\prime|}{1-\gamma}\bigg)\\
   &\qquad+\gamma|p-\hat{p}|\cdot |P^{K}_{U}(y|\ell,u)-P(y|\ell,u)|\frac{B^{\prime}}{1-\gamma} +\gamma E^{+}(B^\prime,\hat{B}^\prime)\\
    &\stackrel{(b)}{\leq}\frac{4(\eta_{\tilde K}+2\beta L_h)}{L_{c}\delta}\bigg(\frac{1+\gamma}{1-\gamma}+\frac{|B^{\prime}-\hat{B}^\prime|}{1-\gamma}\bigg)  +\gamma(|p_1-\hat{p}_1|+|p_2-\hat{p}_2|)\frac{4(\eta_{\tilde K}+2\beta L_h)}{L_{c}\delta}\frac{B^{\prime}}{1-\gamma}+\gamma E^{+}(B^\prime,\hat{B}^\prime)\\
    &\stackrel{(c)}{\leq} \frac{4(\eta_{\tilde K}+2\beta L_h)}{L_{c}\delta}\bigg(\frac{1+\gamma}{1-\gamma}+\frac{|B^{\prime}-\hat{B}^\prime|}{1-\gamma}\bigg)+2\gamma\eta_{\tilde{K}}\frac{4(\eta_{\tilde K}+2\beta L_h)}{L_{c}\delta}\frac{B^{\prime}}{1-\gamma}+\gamma E^{+}(B^\prime,\hat{B}^\prime), 
\end{split}
\end{equation}}
 with probability at least $1-3\lambda\epsilon$, where $(a)$ follows from the definition of $v(\ell,u,y)$ in \eqref{eq:Assu1}, $(b)$ follows from the definition of $P_{U}^{K}(y|\ell,u)$ and Theorem \ref{thm:UCBandLCB}, and 
 $(c)$ follows from  the Theorem \ref{thm:feedbackEstimate}. Also, for $u-\ell\leq \delta$, $ V^*_{\delta}(\ell,u,B^\prime)- \hat{V}^*(\ell,u,B^\prime)$ is zero as both $\delta$-policy and the algorithm  perform action $\ell$. Now, taking expectation with respect to $B^{\prime}$ and $\hat{B}^{\prime}$ and using (\ref{eq:BudgetBound}),  for all $B^\prime\leq B$, we have
\begin{equation}\label{eq:Final}
    E^{+}(B^\prime,\hat{B}^\prime){\leq} \frac{4(\eta_{\tilde K}+2\beta L_h)}{L_{c}\delta (1-\gamma)^2}\bigg({1+\gamma}+\sqrt{B}+\frac{12 B}{(1-p_1-p_2)^2}\frac{(\eta_{\tilde K}+2\beta L_h)}{L_c\cdot \delta } +{\gamma \eta_{\tilde{K}} B^{\prime}}\bigg),
\end{equation}
with probability at least $1-3\lambda\epsilon$ . 
We have 
\begin{equation}\label{eq:regret3}
\begin{split}
    R^+(B)&=G^*(B)-\hat{G}(B)\\
    &\leq E^{+}(B,B) \\
    &\leq \frac{4(\eta_{\tilde K}+2\beta L_h)}{L_{c}\delta (1-\gamma)^2}\bigg({1+\gamma}+\sqrt{B}+\frac{12 B}{(1-p_1-p_2)^2}\frac{(\eta_{\tilde K}+2\beta L_h)}{L_c\cdot \delta } +{\gamma \eta_{\tilde{K}}} B\bigg).
\end{split}
\end{equation}

Likewise, we bound $ R^{++}(B)$. Let for all $B_{r}\leq B$ and $\hat{B}_r$, we have
\begin{equation}
\begin{split}
      E^{++}(B_r,\hat{B}_r)&=\sup_{0\leq \ell\leq u\leq 1} (\hat{V}^{*}(\ell,u,\hat{B}_r)- R(\ell,u,B_r,\hat{y}(\ell,u,\hat{B}_r)),\\
    & =\sup_{0\leq \ell \leq u\leq 1} (\hat{V}^{K}_{U}(\ell,u,\hat{B}_r,\hat{y}(\ell,u,\hat{B}_r))- R(\ell,u,B_r,\hat{y}(\ell,u,\hat{B}_r)),
\end{split}
\end{equation}
where $R(\ell,u,B_r,\hat{y}(\ell,u,\hat B_r))$ is the actual expected reward  from the state $(\ell,u,B_r)$. 
Also,
\begin{equation}
\begin{split}
      E^{++}(B_r)&=\sup_{0\leq \ell\leq u\leq 1} (\hat{V}^{*}(\ell,u,B_r)- R(\ell,u,B_r,\hat{y}(\ell,u,\hat{B}_r)),\\
    & =\sup_{0\leq \ell\leq u\leq 1} (\hat{V}^{K}_{U}(\ell,u,B_r,\hat{y}(\ell,u,B_r))- R(\ell,u,B_r,\hat{y}(\ell,u,\hat{B}_r)),
\end{split}
\end{equation}
Similar to (\ref{eq:maximumErrorBound1.1.1}), we have 
\begin{equation}
    E^{++}(B_r)\leq E^{++}(B_r-x)+x/1-\gamma,
\end{equation}
because for all $B_{r}<0$, $R(\ell,u,B_r,\hat{y}(\ell,u,\hat{B}_r))=\hat{V}^{K}_{U}(\ell,u,B_r,\hat{y}(\ell,u,{B}_r))=0$;  $r(y_{n}(t))\leq 1$ and the expectation is performed with respect to conditional probabilities $P(y|\ell,u)$ and $P_{U}^{K}(y|\ell,u)$ which lies in $[0,1]$; the maximum expected reward in $x$ budget is $x/(1-\gamma)$.
Similar to (\ref{eq:boundValueFunction1.1.1}),  for $u_{1}\leq \ell_{2}$, $B_{1}\leq B_{2}$ and  $y_1,y_2\in [0,1]$, we have
\begin{equation}
    R(\ell_2,u_2,B_2,y_2)-R(\ell_1,u_1,B_1,y_1)\leq \frac{B_2-B_1}{1-\gamma}.
\end{equation}
Using these inequalities, similar to (\ref{eq:regret3}), it can be shown that
\begin{equation}\label{eq:regret2.2.2}
\begin{split}
        &R^{++}(B)\\
        &=\hat{G}(B)-G(B)\\
        &=\hat{V}^{K}_{U}(1,B,\hat{y}(1,B))- R(1,B,\hat{y}(1,B))\\
        &\leq E^{++}(B,B) \leq \frac{4(\eta_{\tilde K}+2\beta L_h)}{L_{c}\delta (1-\gamma)^2}\bigg({1+\gamma}+\sqrt{B}+\frac{12 B}{(1-p_1-p_2)^2}\frac{(\eta_{\tilde K}+2\beta L_h)}{L_c\cdot \delta } +{\gamma \eta_{\tilde{K}}} B\bigg),
\end{split}
\end{equation}
with probability $1-3\lambda\epsilon$.
The statement of the theorem follows from (\ref{eq:MainCheck}), (\ref{eq:regret3})  and (\ref{eq:regret2.2.2}). 
\end{proof}

\subsubsection{Proof of Theorem  \ref{thm:deltaRegret}} \label{sec:thm4}
\begin{proof}
The proof follows from combining Theorem 1 and Theorem 2, replacing the parameters and bounding lower order terms by higher order terms . 
\end{proof}
\subsubsection{Proof of Theorem \ref{thm:lowerBound}} \label{sec:thm5}
\begin{proof}
If $B<1$, then each user abandons the platform as soon as the threshold is crossed, and  the optimal action for both the feedback models is given by \cite{schmit2018learning}
\begin{equation}
    y^*(0,1,B)=\max_{0\leq y\leq 1}(1-F(y))r(y),
\end{equation}
which is also a $\delta$- policy in this scenario as the range of UI $[0,1]$  is greater than $\delta$. 
For each action $y\in [0,1]$, the expected reward is $m(y)=(1-F(y))r(y)$. For all $y_1,y_2\in [0,1]$, we have 
\begin{equation}
\begin{split}
&|m(y_1)-m(y_2)|\\
&=|(1-F(y_1))r(y_1)-(1-F(y_2))r(y_2)|,\\
&=|(1-F(y_1))r(y_1)-(1-F(y_1))r(y_2)+(1-F(y_1))r(y_2)-(1-F(y_2))r(y_2)|,\\
&=|(1-F(y_1))(r(y_1)-r(y_2))+(F(y_2)-F(y_1))r(y_2)|,\\
&\leq L_{r}|y_1-y_2|+L_{h}|y_{2}-y_1|,\\
&=(L_r+L_h)|y_1-y_2|. 
\end{split}
\end{equation}
Thus, $m(y)$ is $(L_{r}+L_h)$-lipshitz continuous. Now, using \cite[Theorem 1]{bubeck2011lipschitz}, we have
\begin{equation}\label{eq:FirstLow}
    \sup_{\mathcal{F}(L)} \mathbf{E}\bigg[\sum_{n=1}^{N}\bigg(\frac{(1-F( y^*(0,1,B)))r(y^*(0,1,B))}{1-\gamma}-\sum_{t=1}^{T_{n}}\gamma^{t-1}r(y_{n}(t))\bigg)\bigg]\geq 0.15 L^{1/3}N^{2/3}.
\end{equation}
For all $B>1$, we have  
\begin{equation}\label{eq:lowValB}
    V^*_{\delta}(0,1,B)\geq (1-F( y^*(0,1,B)))r(y^*(0,1,B)). 
\end{equation}
 Hence, we have
\begin{equation}
\begin{split}
   \sup_{\mathcal{F}(L)} R_{\delta}(N)&=\sup_{\mathcal{F}(L)}\mathbf{E}\bigg[NV^*_{\delta}(0,1,B)-\sum_{n=1}^N\sum_{t=1}^{T_{n}}\gamma^{t-1}r(y_{n}(t))\bigg],\\
    &=NV^*_{\delta}(0,1,B)-\sum_{n=1}^{N}\frac{(1-F( y^*(0,1,B)))r(y^*(0,1,B))}{1-\gamma}\\
    &\qquad+\sup_{\mathcal{F}(L)}\mathbf{E}\bigg[\sum_{n=1}^{N}\bigg(\frac{(1-F( y^*(0,1,B)))r(y^*(0,1,B))}{1-\gamma}-\sum_{t=1}^{T_{n}}\gamma^{t-1}r(y_{n}(t))\bigg)\bigg],\\
    &\stackrel{(a)}{\geq}   0.15 L^{1/3}N^{2/3}.
\end{split}
\end{equation}
where $(a)$ follows from \eqref{eq:FirstLow} and \eqref{eq:lowValB}.
\end{proof}
\subsubsection{Miscellaneous Results of Soft Feedback Model}

{
\begin{theorem}\label{thm:thresholdReveal}  
Let the assumptions in Theorem 1 hold. Let $p^*=\min\{p_1,p_2\}$ and
\begin{equation}
    \Delta=\frac{1-(1-p_2)^{\phi+1}}{B p_2}\cdot \frac{\log_{\tilde{\phi}}(1/\beta)}{1-(p_1+p_2)(1-p^*)^{\phi}}.
\end{equation}
For all $\mathcal{L}$ and $\tilde{\lambda}<1$, we have
\begin{equation}
     P\bigg(K\geq c|\mathcal{L}|\bigg)\geq 1-\lambda.
\end{equation}
where 
\begin{equation}
    c=\frac{(1-\Delta(1+\tilde\lambda))}{(1-\log_{\phi}(1/\beta)/B)}
\end{equation}
\end{theorem}
\begin{proof}
In this proof, we will borrow the notations from the proof of Theorem 1. Let $B^n$ be the random variable representing the budget   used by the platform  of the user $n\in \mathcal{L}$ until  
$J^{th}$ (see \eqref{eq:StopTimePNSF}) iteration.

We first provide a lower bound on the number of times the threshold is crossed for soft feedback model.  Given $\beta$ and  soft feedback model, the algorithm of users in $\mathcal{L}$ can reduce the UI by at most $1/\phi$ times the input UI and cross the threshold at least once in each round of LSE. Hence, the number of rounds of LSE to the reduce the uncertainty $u-\ell \leq \beta$ is at least
\[\frac{\log (1/\beta)}{\log(\phi)}.\]
This implies that  for each user $n\in \mathcal{L}$ such that $u_n-\ell_n\leq \beta$, the number of actions chosen by the platform above the threshold are at least  $\log_{\phi} (1/\beta)$. This implies that the budget utilized for $n\in \mathcal{L}$ such that $u_n-\ell_n\leq \beta$ is
\begin{equation}\label{eq:LowerBoundOnBudget}
    B^n\geq \log_{\phi}(1/\beta).
\end{equation}
Also,  using \eqref{eq:UppBound}, we have
\begin{equation}\label{eq:UpperBoundBudget}
    E[B^n]\leq \Delta B,
\end{equation}
until $u_n-\ell_n\leq \beta$. 
Then, we have
\begin{equation}\label{eq:KConcenNSF}
\begin{split}
&P\bigg(K\leq c|\mathcal{L}|\bigg)\\
&=P(\sum_{n\in \mathcal{L}} \mathbf{1}(B^n< B)\leq |\mathcal{L}|c),\\
&\stackrel{(a)}{\leq} P(\sum_{n\in \mathcal{L}} B^n \geq B(|\mathcal{L}|-c|\mathcal{L}|)+
\log_{\phi}(1/\beta)c |\mathcal{L}|),\\
&\stackrel{}{=} P(\sum_{n\in \mathcal{L}} B^n-|\mathcal{L}|E[B^n]\geq B(|\mathcal{L}|-c|\mathcal{L}|)+
\log_{\phi}(1/\beta) c|\mathcal{L}|-|\mathcal{L}|E[B^n]
),\\
&=P(\sum_{n\in \mathcal{L}} B^n-|\mathcal{L}|E[B^n]\geq B(|\mathcal{L}|-c|\mathcal{L}|)+
\log_{\phi}(1/\beta) c|\mathcal{L}|-|\mathcal{L}|\Delta B+(|\mathcal{L}|\Delta B-|\mathcal{L}|E[B^n])
),\\
&\stackrel{(b)}{\leq} P(\sum_{n\in \mathcal{L}} B^n-|\mathcal{L}|E[B^n]\geq B(|\mathcal{L}|-c|\mathcal{L}|)+
\log_{\phi}(1/\beta) c|\mathcal{L}|-|\mathcal{L}|\Delta B ),\\
&\stackrel{(c)}{=}P(\sum_{n\in \mathcal{L}} B^n-|\mathcal{L}|E[B^n]\geq B \tilde{\lambda}\Delta|\mathcal{L}|),\\
&\stackrel{(d)}{\leq} \exp{\bigg(-2B^2\Delta^2 |\mathcal{L}|^2\tilde\lambda^{ 2}/(B^2|\mathcal{L}|)\bigg)},\\
&= \exp{\bigg(-2 |\mathcal{L}|\Delta^2\tilde\lambda^{ 2}\bigg)},\\
&\stackrel{(e)}{=}{\lambda},
\end{split}
\end{equation}
where $(a)$ follows from \eqref{eq:LowerBoundOnBudget}  and the fact that $B^n=B$ for users who have abandoned the platform, 
$(b)$ follows from \eqref{eq:UpperBoundBudget},  $(c)$ follows from by replacing the value of $c$, $(d)$ follows from Hoeffding's inequality and the fact that random variable $B^n\in [0,B]$, and $(e)$ follows from the definition of $\tilde{\lambda}$.
\end{proof}
}
\begin{theorem}\label{thm:UCBandLCB}
Let the assumptions in Theorem 1 hold. For all $\mathcal{L}$, $\tilde{\lambda}<1$ and $\ell,u,y\in[0,1]$ such that $u-\ell>\delta$, we have
\begin{equation}\label{eq:UCBLCBbound}
    P_{L}^{\tilde{K}}(y|\ell,u)\leq P(y|\ell,u) \leq  P_{U}^{\tilde{K}}(y|\ell,u),
\end{equation}
with probability at least $1-\lambda\epsilon$, where $\tilde{K}=c|\mathcal{L}|$ .
\end{theorem}
\begin{proof}
{First, we show that given $K$, we have 
\begin{equation}
    P_{L}^{{K}}(y|\ell,u)\leq P(y|\ell,u) \leq  P_{U}^{{K}}(y|\ell,u),
\end{equation} 
with probability at least $1-\epsilon$.
Second, using Theorem \ref{thm:thresholdReveal}, we have $K\geq c|\mathcal{L}|$ with probability at least $1-\lambda$. Since confidence interval $2(\eta_{K}+2\beta L_{h})/(L_c\delta)$ is a decreasing function in $K$, the theorem follows by combining first and second step.}

If $ P_{U}^{K}(y|\ell,u)\leq  P(y|\ell,u)$ and $\hat{F}^{K}(u)-\hat{F}^{K}(\ell)=\max\{\hat{F}^{K}(u)-\hat{F}^{K}(\ell),L_c(u-\ell)\}$, then
\begin{equation}
    \frac{\hat{F}^{K}(u)-\hat{F}^{K}(y)}{\hat{F}^{K}(u)-\hat{F}^{K}(\ell)}+\frac{2(\eta_{K}+2\beta L_{h})}{L_{c}\delta} \leq \frac{F(u)-F(y)}{F(u)-F(\ell)}.
\end{equation}
Therefore, at least one of the following events $E_1$ and $E_2$ is true
\[E_{1}:\hat{F}^{K}(u)-\hat{F}^{K}(y)\leq F(u)-F(y)-\eta_{K}-2\beta L_{h},\]
\[E_{2}:\hat{F}^{K}(u)-\hat{F}^{K}(\ell)\geq F(u)-F(\ell)+\eta_{K}+2\beta L_{h}.\]
This can be proved by contradiction. Let both $E_{1}$ and $E_{2}$ are false. Let $\rho(u,\ell)=F(u)-F(\ell)$ and $\Bar{\rho}^{K}(u,\ell)=\hat{F}^{K}(u)-\hat{F}^{K}(\ell)$. Then, we have 
\begin{equation}
\begin{split}
\frac{\rho(u,y)}{\rho(u,\ell)}-\frac{\Bar{\rho}^{K}(u,y)}{\Bar{\rho}^{K}(u,\ell)}&= \frac{\rho(u,y)\Bar{\rho}^{K}(u,\ell)-\rho(u,\ell)\Bar{\rho}^{K}(u,y)}{\rho(u,\ell)\Bar{\rho}^{K}(u,\ell)},\\
&=\frac{\rho(u,y)(\Bar{\rho}^{K}(u,\ell)-\rho(u,\ell))+\rho(u,\ell)(\rho(u,y)-\Bar{\rho}^{K}(u,y))}{\rho(u,\ell)\Bar{\rho}^{K}(u,\ell)},\\
&\stackrel{(a)}{\leq} \frac{\rho(u,y)(\eta_{K}+2\beta L_{h})+\rho(u,\ell)(\eta_{K}+2\beta L_{h})}{\rho(u,\ell)\Bar{\rho}^{K}(u,\ell)},\\
&\stackrel{(b)}{\leq}\frac{\eta_{K}+2\beta L_{h}}{\Bar{\rho}^{K}(u,\ell)}+\frac{\eta_{K}+2\beta L_{h}}{\Bar{\rho}^{K}(u,\ell)},\\
&\stackrel{(c)}{\leq} \frac{2(\eta_{K}+2\beta L_{h})}{L_{c}\delta},
\end{split}
\end{equation}
where $(a)$ follows from the fact that both $E_{1}$ and $E_{2}$ are false,  $(b)$ follows from the fact that $p(y|\ell,u)\leq 1$, and $(c)$ follows from the facts that $\hat{F}^{K}(u)-\hat{F}^{K}(\ell)=\max\{\hat{F}^{K}(u)-\hat{F}^{K}(\ell),L_c(u-\ell)\}$, $u-\ell>\delta$, and Assumption 2. Hence, at least one of the events $E_{1}$ and $E_{2}$ is true. 

Now, we bound the probability of $E_1$ and $E_2$. 
According to Dvoretzky–Kiefer–Wolfowitz inequality (\cite{dvoretzky1956asymptotic}), we have 
\begin{equation}\label{eq:DKWbound}
    \mathbf{P}(\sup_{x}|{\hat{G}^{m}(x)-G(x)}|\geq \Tilde{\delta})\leq \exp{(-\Tilde{\delta}^2 m)},
\end{equation}
where $G$ is any cumulative distribution function and its estimate $\Hat{G}^{m}(x)$ is
\begin{equation}
    \Hat{G}^{m}(x)=\sum_{t=1}^{m}\frac{\mathbf{1}(X_{t}\leq x)}{m},
\end{equation}
and for all $t$, $X_{t}$ are  independent and identically distributed random variables drawn from $G$. Also, 
\begin{equation}\label{eq:rangeOfEstimator1.1}
\begin{split}
    \frac{\sum_{n\in\mathcal{L}}\mathbf{1}(u_n-\ell_n\leq \beta)\mathbf{1}(\ell_n\leq x)}{K}\in& \Bigg\{\frac{\sum_{n\in\mathcal{L}}\mathbf{1}(u_n-\ell_n\leq \beta)\mathbf{1}(\theta_n\leq x)}{K}, \frac{\sum_{n\in\mathcal{L}}\mathbf{1}(u_n-\ell_n\leq \beta)\mathbf{1}(\theta_n\leq x)}{K}\\
    &+ \frac{\sum_{n\in\mathcal{L}}\mathbf{1}(u_n-\ell_n\leq \beta)\mathbf{1}(\ell_n\leq x\leq \theta_n )}{K}\Bigg\}.
\end{split}
\end{equation}

Using these, we have 
\begin{equation}\label{eq:DiffEstimate}
\begin{split}
    &P(\max_{\ell,u}|\Bar{\rho}^{K}(u,\ell)-{\rho}(u,\ell)|\geq \eta_{K}+2\beta L_h)\\
    &\leq P(\max_{\ell}|\hat{F}^{K}(\ell)-{F}(\ell)|+\max_{u}|\hat{F}^{K}(u)-{F}(u)|\geq \eta_{K}+2\beta L_h ) ,\\
    &\stackrel{(a)}{\leq} 2 P(\max_{\ell}|\frac{\sum_{n\in \mathcal{L}} \mathbf{1}(u_n-\ell_n\leq \beta)\mathbf{1}(\theta_n\leq \ell)}{K}-{F}(\ell)|\geq \eta_{K}/3)\\
    &\qquad\qquad+P\Bigg(2\frac{\sum_{n\in\mathcal{L}}\mathbf{1}(u_n-\ell_n\leq \beta)\mathbf{1}(\ell_n\leq x\leq \theta_n )}{K}\geq \eta_{K}/3+2\beta L_{h}\Bigg),\\
    &\stackrel{(b)}{\leq} 2\exp(-\eta^2_{K}K/9) + 2\exp(-2 \eta^2_{K}K/36) \leq \epsilon/4.
\end{split}
\end{equation}
where $(a)$ follows from (\ref{eq:rangeOfEstimator1.1}) and the union bound, and $(b)$ follows from the fact that for all $u_n-\ell_n \leq \beta$ which implies $E[\mathbf{1}(\ell_n\leq x\leq \theta_n )]\leq \beta L_h$,  Hoeffding's inequality and Dvoretzky–Kiefer–Wolfowitz inequality. 

Thus, the $P(E_1)$ and $P(E_{2})$ are at most $\epsilon/4$. Hence, if $\hat{F}^{K}(u)-\hat{F}^{K}(\ell)=\max\{\hat{F}^{K}(u)-\hat{F}^{K}(\ell),L_c(u-\ell)\}$, then
\begin{equation}\label{eq:U1}
    P\bigg(P_{U}^{K}(y|\ell,u)\leq  P(y|\ell,u)\bigg)\leq P(E_1)+P(E_2)\leq \epsilon/2.
\end{equation}

Now, if $ P_{U}^{K}(y|\ell,u)\leq  P(y|\ell,u)$ and $L_c(u-\ell)=\max\{\hat{F}^{K}(u)-\hat{F}^{K}(\ell),L_c(u-\ell)\}$, then
\begin{equation}
    \frac{\hat{F}^{K}(u)-\hat{F}^{K}(y)}{L_c(u-\ell)}+\frac{2(\eta_{K}+2\beta L_h)}{L_{c}\delta} \leq \frac{F(u)-F(y)}{F(u)-F(\ell)}.
\end{equation}
Since $u-\ell>\delta$, the above equation implies
\begin{equation}
   {\hat{F}^{K}(u)-\hat{F}^{K}(y)}+{2(\eta_{K}+2\beta L_h)} \leq {F(u)-F(y)},
\end{equation}
and the probability of this event is at most $\epsilon/4$
using (\ref{eq:DiffEstimate}). Thus, if $L_c(u-\ell)=\max\{\hat{F}^{K}(u)-\hat{F}^{K}(\ell),L_c(u-\ell)\}$, then
\begin{equation}\label{eq:U2}
    P\bigg(P_{U}^{K}(y|\ell,u)\leq  P(y|\ell,u)\bigg)\leq  \epsilon/4.
\end{equation}
Combining (\ref{eq:U1}) and (\ref{eq:U2}), we have 
\[P\bigg(P_{U}^{K}(y|\ell,u)\leq  P(y|\ell,u)\bigg)\leq \epsilon/2.\]
Similarly, we can show that
\[P\bigg(P_{L}^{K}(y|\ell,u)\geq  P(y|\ell,u)\bigg)\leq  \epsilon/2.\]
Hence, the (\ref{eq:UCBLCBbound}) holds with probability at least $1-\epsilon$. 

\end{proof}
\begin{theorem}\label{thm:feedbackEstimate}
Let the assumptions in Theorem 1 hold. For $0<p_1<1$, we have 
\begin{equation}\label{eq:p_1estimate}
    \hat{p}_1-\eta_{\tilde{K}}\leq p_1 \leq \hat{p}_1+\eta_{\tilde{K}},
\end{equation}
with probability at least $1-\lambda\epsilon/8$, where $\tilde{K}=c |\mathcal{L}|$. Likewise, 
for $0<p_2<1$, we have 
\begin{equation}\label{eq:p_2estimate}
    \hat{p}_2-\eta_{\tilde{K}}\leq p_2 \leq \hat{p}_2+\eta_{\tilde{K}},
\end{equation}
with probability at least $1-\lambda\epsilon/8$.
\end{theorem}
\begin{proof}
{First, we show that given $K$, we have 
\begin{equation}
    \hat{p}_1-\eta_{{K}}\leq p_1 \leq \hat{p}_1+\eta_{{K}},
\end{equation}
with probability at least $1-\epsilon/8$.
Second, using Theorem \ref{thm:thresholdReveal}, we have $K\geq c|\mathcal{L}|$ with probability at least $1-\lambda$. Since confidence interval $\eta_{K}$ is a decreasing function in $K$, \eqref{eq:p_1estimate} follows by combining first and second step.} 
Similarly, we can show that \eqref{eq:p_2estimate} holds. 

We have 
\begin{equation}
\begin{split}
    |\hat{p}_1-p_1|&=\bigg|\frac{\sum_{n\in \mathcal{L}}\sum_{t=1}^{T(n)}\textbf{1}(u_n-\ell_n\leq \beta)\mathbf{1}(S(y_{n}(t))=1)}{\sum_{n\in \mathcal{L}}\sum_{t^\prime=1}^{T(n)}\textbf{1}(u_{n^\prime}-\ell_{n^\prime}\leq \beta)\mathbf{1}( y_{n^\prime}(t^\prime)\leq \ell_{n^{\prime}})}-p_1\bigg|,\\
    &\stackrel{(a)}{=}\bigg|\frac{\sum_{n\in \mathcal{L}}\sum_{t=1}^{T(n)}\textbf{1}(u_n-\ell_n\leq \beta)\mathbf{1}( y_{n}(t)\leq \ell_n)\mathbf{1}(S(y_{n}(t))=1)}{\sum_{n\in \mathcal{L}}\sum_{t^\prime=1}^{T(n)}\textbf{1}(u_{n^\prime}-\ell_{n^\prime}\leq \beta)\mathbf{1}( y_{n^\prime}(t^\prime)\leq \ell_{n^{\prime}})}-p_1\bigg|,
\end{split}
\end{equation}
where $(a)$ follows from the fact that for all $y_n(t)$ such that $S(y_{n}(t))=1$, we have $ y_n(t)\leq \ell_n$.  Now, we have
\begin{equation}
    \sum_{n\in \mathcal{L}}\sum_{t^\prime=1}^{T(n)}\textbf{1}(u_{n^\prime}-\ell_{n^\prime}\leq \beta)\mathbf{1}( y_{n^\prime}(t^\prime)\leq \ell_{n^\prime})\geq K,
\end{equation}
since $\theta_n\in (0,1)$ and there exists at least one action $y_{n}(t)=0$ below the threshold. 

Using Hoeffding's inequality, we have 
\begin{equation}
    P(|\hat{p}_1-p_1|\geq \eta_{{K}})\leq \epsilon/8.
\end{equation}

We have 
\begin{equation}
\begin{split}
    |\hat{p}_2-p_2|&=\bigg|\frac{\sum_{n\in \mathcal{L}}\sum_{t=1}^{T(n)}\textbf{1}(u_n-\ell_n\leq \beta)\mathbf{1}(S(y_{n}(t))=0)}{\sum_{n\in \mathcal{L}}\sum_{t^\prime=1}^{T(n)}\textbf{1}(u_{n^\prime}-\ell_{n^\prime}\leq \beta)\mathbf{1}(u_{n^{\prime}}\leq y_{n^\prime}(t^\prime))}-p_2\bigg|,\\
    &\stackrel{(a)}{=}\bigg|\frac{\sum_{n\in \mathcal{L}}\sum_{t=1}^{T(n)}\textbf{1}(u_n-\ell_n\leq \beta)\mathbf{1}(u_{n}\leq  y_{n}(t))\mathbf{1}(S(y_{n}(t))=0)}{\sum_{n\in \mathcal{L}}\sum_{t^\prime=1}^{T(n)}\textbf{1}(u_{n^\prime}-\ell_{n^\prime}\leq \beta)\mathbf{1}(u_{n^{\prime}}\leq y_{n^\prime}(t^\prime))}-p_2\bigg|,
\end{split}
\end{equation}
where $(a)$ follows from the fact that for all $y_n(t)$ such that $S(y_{n}(t))=0$, we have $u_{n}\leq y_n(t)$. Now, we have
\begin{equation}
    \sum_{n\in \mathcal{L}}\sum_{t^\prime=1}^{T(n)}\textbf{1}(u_{n^\prime}-\ell_{n^\prime}\leq \beta)\mathbf{1}(u_{n^{\prime}}\leq y_{n^\prime}(t^\prime))\geq K,
\end{equation}
since $\theta_n\in (0,1)$ and there exists at least one action $y_{n}(t)=1$ above the threshold. 
Using Hoeffding's inequality, we have 
\begin{equation}
    P(|\hat{p}_2-p_2|\geq \eta_{{K}})\leq \epsilon/8.
\end{equation}
Hence, the statement of the theorem follows. 
\end{proof}
\subsection{Hard Feedback Model}
\begin{theorem}\label{thm:explorationHF}  For all $\delta\geq 0$,  $B$ such that $\log_{\phi}(1/\beta)+1\leq B$, $\phi>1$, the $\delta$-regret of the platform  over the exploration set $\mathcal{L}$ in UCB-PVI-HF for hard feedback model is
\begin{equation}
\begin{split}
    R_{\delta}(|\mathcal{L}|)&\leq |\mathcal{L}|\bigg(     (\phi+1)L_{r}\log_{{\phi}}(\log_{{\phi}}(1/{\beta}) +1)+\log_{{\phi}}(1/{\beta})+1 +(\phi+1)L_r 
+  \frac{\beta L_r}{1-\gamma}\bigg).
\end{split}
\end{equation}
\end{theorem}
\begin{proof} 
Letting $[\ell(j), u(j)]$ be the UI after  $j$ rounds of Linear Search Exploration and $I_{j}=u(j)-\ell(j)$. Let $J$ be the first instance at which $u(J)-\ell(J)\leq \beta$, namely
\begin{equation}\label{eq:StopTimeP}
    J=\min\{j: u(j)-\ell(j)\leq \beta\}.
\end{equation}

{Since $K$ is the number of users in $\mathcal{L}$ that do not abandon the platform, we have $J<\infty$ with probability one for these users.}

The $\delta$-regret for these users is  
\begin{equation}\label{eq:newRegret1}
    R_{\delta}(K)\leq \sum_{n=1}^{K}\mathbf{E}\bigg[\sum_{t=1}^{\infty}( \gamma^{t-1}r(\theta_{n})-r(y_{n}(t))\bigg|\theta_{n}\bigg],
\end{equation}
because given $\theta_n$, $r(\theta_{n})$ is the maximum achievable reward for a user and the platform and $\gamma<1$. 


{Assuming $J<\infty$, namely there exists a $J$ such that $u(J)-\ell(J)\leq \beta$.} We divide the regret analysis into two regions: until round $J$ and after round $J$.

 For all $j\leq J$, the residual uncertainty $I_{j}$ at round $j$ is reduced by a factor of $1/\phi$ compared to the residual uncertainty $I_{j-1}$ at the previous round for hard feedback case, namely
\begin{equation}\label{eq:redUI}
I_{j}=\frac{I_{j-1}}{\phi},   
\end{equation}
because given $I_{j}$ at round $j$, the actions in set $A=\{\ell(j-1), \ell(j-1)+I_{j},\ell(j-1)+2I_{j},\ldots, u(j-1)\}$ are performed sequentially at this round, and feedback is received for each action. The regret, as defined in (\ref{eq:newRegret1}),  at round $j$ of LSE is at most
\[L_{r}I_{j-1}(\phi+1)+1,\]
because 
the platform performs $\phi +1$ actions in LSE and the threshold can be crossed at most one time, which contributes to unity in the above expression. Combining the fact $I_{0}=1$ and \eqref{eq:redUI}, we have
\begin{equation}\label{eq:PHardFeedbakc}
    J\leq \log_{{\phi}}(1/{\beta}) +1.
\end{equation}
The  budget utilized until $J$ is at most 
\begin{equation}\label{eq:UppBudnget1}
    \log_{{\phi}}(1/{\beta})+1,
\end{equation}
and 
$\log_{{\phi}}(1/{\beta})+1 \leq B$ using the assumption in the theorem. Thus, $J$ is finite for all users in $\mathcal{L}$, namely $K=|\mathcal{L}|$. 

Thus, the total regret until round $J$ of LSE is
{\small
\begin{equation}\label{eq:1}
\begin{split}
   &\sum_{j=1}^{J} \bigg(\big(1+\phi\big)L_r I_{j-1}+ 1\bigg),\\
   &\leq \sum_{j=1}^{J} \bigg(\frac{(1+\phi)L_r}{\phi^{j-1}}+ 1\bigg),\\
   &\stackrel{}{\leq} (1+\phi)L_r+\sum_{j=2}^{J} \frac{(\phi+1)L_{r}}{ {(j-1)}}+ J,\\
   &\leq  (\phi+1)L_{r}\log_{\phi}(J)+J+(\phi+1)L_r, \\
   &\leq (\phi+1)L_{r}\log_{\phi}(\log_{{\phi}}(1/{\beta}) +1)+\log_{{\phi}}(1/{\beta}) +1+(\phi+1)L_r,
\end{split}
\end{equation}}
where last inequality follows from \eqref{eq:PHardFeedbakc}. 
The regret following $J$ iterations of LSE is at most
\begin{equation}\label{eq:2}
\beta L_{r}/(1-\gamma).
\end{equation}


Combining (\ref{eq:1}) and (\ref{eq:2}), the regret for each user $n\in\mathcal{L}$ is at most
\begin{equation}
\begin{split}
&\mathbf{E}\bigg[\sum_{t=1}^{\infty}( \gamma^{t-1}r(\theta_{n})-r(y_{n}(t))\bigg|\theta_{n}\bigg]\\
     &\leq   (\phi+1)L_{r}\log_{{\phi}}(\log_{{\phi}}(1/{\beta}) +1)+\log_{{\phi}}(1/{\beta})+1 +(\phi+1)L_r 
+  \frac{\beta L_r}{1-\gamma}.\\
\end{split}
\end{equation}
The statement of the theorem now follows by summing over all the users in the exploration set.
\end{proof}
\subsubsection{Proof of Theorem \ref{thm:HFexploitation}}\label{sec:thm3}
\begin{proof} 

Let $G^*(B)=V^*_{\delta}(0,1,B)$ be the  expected reward  received by the $\delta$-policy for patience budget $B$ for each user. 
Given the estimate $\hat{F}^{K}(x)$ and $B_r$,  $\hat{G}(B)=\hat{V}^*(0,1,B)$ is the estimate of expected rewards using UCB-PVI-HF for each user in $\mathcal{E}$.   $G(B)$ 
 is the actual expected rewards  of the platform's policy in UCB-PVI-HF for each user in $\mathcal{E}$. In hard feedback model, $\hat B_r=B_r$ as the $B_r$ can be perfectly estimated from the feedback. In the notations, we drop the parameter $\delta$ for simplicity.

Thus, the cumulative regret over all users in $\mathcal{E}$ with respect to $\delta$-policy is 
\begin{equation}
\begin{split}
        R_{\delta}(E) &= |\mathcal{E}|( G^*(B)-G(B))=|\mathcal{E}|(G^*(B)-\hat{G}(B)+\hat{G}(B)-G(B)).
\end{split}
\end{equation}
In the above equation, we bound the two terms $R^{+}(B)=G^*(B)-\hat{G}(B)$ and $R^{++}(B)=\hat{G}(B)-G(B)$ individually. 

For $R^{+}(B)$, given $B_{r}\leq B$, we define 
\begin{equation}
\begin{split}
     E^{+}(B_{r})&= \sup_{0\leq\ell\leq u\leq 1} \big( {V}^*_{\delta}(\ell,u,B_{r})-\hat{V}^*(\ell,u,B_{r})\big).
\end{split}
\end{equation}
Given $B_r$, $E^{+}(B_{r})$ is the maximum error in estimation of value function for a remaining budget $B_{r}$, which is observable in hard feedback model.
For any $B_{r}\leq B$ and $x\geq 0$, the following holds 
{\small
\begin{equation}\label{eq:maximumErrorBound}
\begin{split}
     E^{+}(B_{r})&=\sup_{0\leq\ell\leq u\leq 1} \big( {V}^*_{\delta}(\ell,u,B_{r})-\hat{V}^*(\ell,u,B_{r})\big)\\
     &\stackrel{(a)}{\leq} \sup_{0\leq\ell\leq u\leq 1} \bigg( {V}^*_{\delta}(\ell,u,B_{r}-x)-\hat{V}^*(\ell,u,B_{r}-x) + \frac{x}{1-\gamma}\bigg)\\
     &=E^{+}(B_{r}-x)+\frac{x}{1-\gamma},
\end{split}
\end{equation}}
where $(a)$ follows from the following facts:  for all $B_{r}< 0$, $V^*_{\delta}(\ell,u,B_{r})=\hat{V}(\ell,u,B_r)=0$; $r(y_{n}(t))\leq 1$ and the expectation is performed with respect to conditional probabilities $P(y|\ell,u)$ and $P_{U}^{K}(y|\ell,u)$ which lies in $[0,1]$; the maximum reward due to  platform-user interactions in $x$ budget is $x/(1-\gamma)$.
 Also, using the fact that the rewards  function $r(y)$ is non-decreasing in $y$,  for $u_{1}\leq \ell_{2}$ and $B_{1}\leq B_{2}$, we have 
 {\small
\begin{equation}\label{eq:boundValueFunction}
    V^*_{\delta}(\ell_2,u_2,B_{2})-V^*_{\delta}(\ell_1,u_1,B_{1})\leq \frac{B_{2}-B_{1}}{1-\gamma}.
\end{equation}
Likewise, 
\begin{equation}
    \hat{V}^*(\ell_2,u_2,B_{2})- \hat{V}^*(\ell_1,u_1,B_{1})\leq \frac{B_{2}-B_{1}}{1-\gamma},
\end{equation}
Next, we  bound $R^{+}(B)=G^*(B)-\hat{G}(B)$ in terms of $E^{+}(B)$. Now, suppose 
\begin{equation}\label{eq:definePU}
    P_{U}^{K}(y|\ell,u)=P(y|\ell,u)+v(\ell,u,y).
\end{equation}}
For all $\ell,u\in [0,1]$ such that $\ell-u>\delta$ and $B^\prime\leq B$, the lower bound on $\hat{V}^*(\ell,u,B^\prime)$ is as follows:
{\small
\begin{equation}
\begin{split}
    &\hat{V}^*(\ell,u,B^\prime)\\
    &\stackrel{(a)}{\geq}  \hat{V}^{K}_{U}(\ell,u,B^\prime,y^*_{\delta}(\ell,u,B^\prime)),\\
                 &= P_{U}^{K}\big(y^*_{\delta}(\ell,u,B^\prime)|\ell,u\big)\big(r({y}^*_{\delta}(\ell,u,B^\prime))+\gamma\hat{V}^*({y}^*_{\delta}(\ell,u,B^\prime),u,B^\prime)\big),\\
                 &\qquad +\big(1-P_{U}^{K}(y^*_{\delta}(\ell,u,B^\prime)|\ell,u)\big)\gamma\hat{V}^*(\ell,y^*_{\delta}(\ell,u,B^\prime),B^\prime-1)\\
       &\stackrel{(b)}{=} \bigg(P\big({y}^*_{\delta}(\ell,u,B^\prime)|\ell,u\big)+v(\ell,u,y^*_{\delta}(\ell,u,B^\prime))\bigg)\big(r({y}^*_{\delta}(\ell,u,B^\prime))+\gamma\hat{V}^*({y}^*_{\delta}(\ell,u,B^\prime),u,B^\prime)\big)\\
       &\quad+\bigg(1-P\big({y}^*_{\delta}(\ell,u,B^\prime)|\ell,u\big)-v(\ell,u,y^*_{\delta}(\ell,u,B^\prime))\bigg)\gamma\hat{V}^*(\ell,y^*_{\delta}(\ell,u,B^\prime),B^\prime-1)\\
       &\stackrel{(c)}{\geq}  \bigg(P\big({y}^*_{\delta}(\ell,u,B^\prime)|\ell,u\big)+v(\ell,u,y^*_{\delta}(\ell,u,B^\prime))\bigg)\bigg(r({y}_{\delta}^*(\ell,u,B^\prime))+\gamma{V}^*_{\delta}({y}^*_{\delta}(\ell,u,B^\prime),u,B^\prime)\\
       &-\gamma E^{+}(B^\prime)\bigg)+\bigg(1-P\big({y}_{\delta}^*(\ell,u,B^\prime)|\ell,u\big)-v(\ell,u,y_{\delta}^*(\ell,u,B^\prime))\bigg)\bigg(\gamma{V}^*_{\delta}(\ell,y^*_{\delta}(\ell,u,B^\prime),B^{\prime}-1)\\
       &\qquad -\gamma E^{+}(B^{\prime}-1)\bigg)\\
       &\stackrel{(d)}{\geq}  V^*_{\delta}(\ell,u,B^\prime) - \gamma E^{+}(B^\prime)+ v(\ell,u,y^*_{\delta}(\ell,u,B^\prime))r(y^*_{\delta}(\ell,u,B^\prime))\\
       &\qquad-v(\ell,u,y^*_{\delta}(\ell,u,B^\prime))\gamma\big(E^{+}(B^\prime)-E^{+}(B^\prime-1)\big)\\
       &\qquad+ v(\ell,u,y^*_{\delta}(\ell,u,B^\prime))\gamma\bigg({V}^*_{\delta}({y}^*_{\delta}(\ell,u,B^\prime),u,B^\prime)-{V}^*_{\delta}(\ell,y^*_{\delta}(\ell,u,B^\prime),B^\prime-1)\bigg)\\
       &\stackrel{(e)}{\geq} V^*_{\delta}(\ell,u,B^\prime)-|v(\ell,u,y^*_{\delta}(\ell,u,B^\prime))|\bigg(1+2\frac{\gamma}{1-\gamma}\bigg)-\gamma E^{+}(B^\prime),\\
\end{split}
\end{equation}
}
where $(a)$ follows from the fact that the optimal action according to $\hat{V}^*(\ell,u,B^\prime)$ is $\hat{y}(\ell,u,B^\prime)$; $(b)$  follows from \eqref{eq:definePU}; $(c)$ follows from the definition of $E^+(B^\prime)$, $(d)$ follows from the fact that the probability lies in $[0,1]$ and $E^{+}(B^\prime)\geq E^{+}(B^\prime-1)$,because  the instance or the optimal policy that corresponds to value of $ E^{+}(B^\prime-1)$ is also be  feasible in budget $B^\prime$; $(e)$ follows from (\ref{eq:maximumErrorBound}) and (\ref{eq:boundValueFunction}). 
From the above equation,  we have that  
\begin{equation}\begin{split}\label{eq:DiffBetwVFnEstimate_1}
   V^*_{\delta}(\ell,u,B^\prime)- \hat{V}^*(\ell,u,B^\prime)&\leq |v(\ell,u,y^*(\ell,u,B^\prime))|\bigg(\frac{1+\gamma}{1-\gamma}\bigg)+\gamma E^{+}(B^\prime),\\
   & \stackrel{(a)}{\leq} \sup_{\ell,u,y} |P^{K}_{U}(y|\ell,u)-P(y|\ell,u)|\bigg(\frac{1+\gamma}{1-\gamma}\bigg) +\gamma E^{+}(B^\prime),\\
    &\stackrel{(b)}{\leq} \frac{4(\eta_{K}+2\beta L_h)}{L_{c}\delta}\bigg(\frac{1+\gamma}{1-\gamma}\bigg)+\gamma E^{+}(B^\prime),
\end{split}
\end{equation}
with probability at least $1-\epsilon$, 
where $(a)$ follows from the definition of $v(\ell,u,y)$ in \eqref{eq:definePU}, $(b)$ follows from the definition of $P_{U}^{K}(y|\ell,u)$ and 
Theorem \ref{thm:UCBLCB_HF}.
Also, for $u-\ell\leq \delta$, we have that $ V^*_{\delta}(\ell,u,B^\prime)- \hat{V}^*(\ell,u,B^\prime)$ is zero as both $\delta$-policy and the algorithm perform action $\ell$. 
In conclusion, we have that for all $B^{\prime}\leq B$
\begin{equation}\label{eq:Final_1}
    E^{+}(B^\prime){\leq} \frac{4(\eta_{K}+2\beta L_h)}{L_{c}\delta}\bigg(\frac{1+\gamma}{(1-\gamma)^2}\bigg),
\end{equation}
with probability at least $1-\epsilon$ . 
We have 
\begin{equation}\label{eq:regret1_HF}
    R^+(B)=G^*(B)-\hat{G}(B) \leq E^{+}(B) \leq \frac{4(\eta_{K}+2\beta L_h)}{L_{c}\delta}\bigg(\frac{1+\gamma}{(1-\gamma)^2}\bigg).
\end{equation}

Likewise, we bound $ R^{++}(B)$. Let for all $B_{r}\leq B$,
\begin{equation}
\begin{split}
      E^{++}(B_r)&=\sup_{0\leq \ell\leq u\leq 1} (\hat{V}^{*}(\ell,u,B_r)- R(\ell,u,B_r,\hat{y}(\ell,u,B_r)),\\
    & =\sup_{0\leq \ell\leq u\leq 1} (\hat{V}^{K}_{U}(\ell,u,B_r,\hat{y}(\ell,u,B_r))- R(\ell,u,B_r,\hat{y}(\ell,u,B_r)),
\end{split}
\end{equation}
where $R(\ell,u,B_r,\hat{y}(\ell,u,B_r))$ is the actual expected reward  from the state $(\ell,u,B_r)$ using UCB-PVI-HF. 
Similar to (\ref{eq:maximumErrorBound}), we have 
\begin{equation}
    E^{++}(B_r)\leq E^{++}(B_r-x)+x/1-\gamma,
\end{equation}
because for all $B_{r}<0$, $R(\ell,u,B_r,\hat{y}(\ell,u,B_r))=\hat{V}^{K}_{U}(\ell,u,B_r,\hat{y}(\ell,u,B_r))=0$;  $r(y_{n}(t))\leq 1$ and the expectation is performed with respect to conditional probabilities $P(y|\ell,u)$ and $P_{U}^{K}(y|\ell,u)$ which lies in $[0,1]$; the maximum expected reward in  budget $x$ is $x/(1-\gamma)$.
Similar to (\ref{eq:boundValueFunction}),  for $u_{1}\leq \ell_{2}$, $B_{1}\leq B_{2}$ and  $y_1,y_2\in [0,1]$, we have
\begin{equation}
    R(\ell_{2},u_2,B_2,y_2)-R(\ell_1,u_1,B_1,y_1)\leq \frac{B_2-B_1}{1-\gamma}.
\end{equation}
Using these inequalities, similar to (\ref{eq:regret1_HF}), it can be shown that
\begin{equation}\label{eq:regret2_HF}
\begin{split}
        R^{++}(B)&=\hat{G}(B)-G(B),\\
        &=\hat{V}^{K}_{U}(0,1,B,\hat{y}(0,1,B))- R(0,1,B,\hat{y}(0,1,B)),\\
        &\leq E^{++}(B) \leq \frac{4(\eta_{K}+2\beta L_h)}{L_{c}\delta}\bigg(\frac{1+\gamma}{(1-\gamma)^2}\bigg).
\end{split}
\end{equation}
The statement of the theorem  follows from \eqref{eq:regret1_HF} and \eqref{eq:regret2_HF}. 
\end{proof}

\begin{theorem}\label{thm:deltaHF} For {$|\mathcal{L}|= (\log(16/\epsilon))^{1/2}N^{2/3}(1-\tilde{\lambda})^{1/2}/( B^{1/3}) $} and $\beta = \eta_{|\mathcal{L}|}/2 L_h$, 
there exists a constant $M$ such that  with probability at least $1-\epsilon$, the $\delta$-Regret  of UCB-PVI-HF for hard feedback model is $\tilde{O}(N^{2/3})$. 
\end{theorem}
\begin{proof}
The proof follows from Theorem \ref{thm:explorationHF} and Theorem 3.
\end{proof}
\subsubsection{Miscellaneous result}
\begin{theorem}\label{thm:UCBLCB_HF}
Let the assumptions in Theorem \ref{thm:explorationHF} hold. For all $\mathcal{L}$, and $\ell,u,y\in[0,1]$ such that $u-\ell>\delta$, we have
\begin{equation}\label{eq:UCBLCBbound_HF}
    P_{L}^{|\mathcal{L}|}(y|\ell,u)\leq P(y|\ell,u) \leq  P_{U}^{|\mathcal{L}|}(y|\ell,u),
\end{equation}
with probability at least $1-\epsilon$.
\end{theorem}
\begin{proof}
Using \eqref{eq:UCBLCBbound}, given $K$, we have 
\begin{equation}
    P_{L}^{{K}}(y|\ell,u)\leq P(y|\ell,u) \leq  P_{U}^{{K}}(y|\ell,u),
\end{equation} 
with probability at least $1-\epsilon$. Using \eqref{eq:UppBudnget1} and the assumption that $\log_{{\phi}}(1/{\beta})+1 \leq B$, in hard feedback model, we have
\begin{equation}
    K=|\mathcal{L}|. 
\end{equation}
The statement of the theorem follows.
\end{proof}

\vskip 0.2in
\bibliographystyle{theapa}
\bibliography{refereces}

\end{document}